\documentclass[letterpaper]{article} 
\usepackage{aaai24}
\usepackage{times}  
\usepackage{helvet}  
\usepackage{courier}  
\usepackage[hyphens]{url}  
\usepackage{graphicx} 
\usepackage{natbib}  
\usepackage{caption} 
\usepackage{booktabs}
\usepackage{algorithm}
\usepackage{algorithmic}
\usepackage{amsmath}
\usepackage{amssymb}
\usepackage{subfig}
\usepackage{xcolor}
\usepackage{amsthm}

\newtheorem{lemma}{Lemma}

\newtheorem{definition}{Definition}

\newtheorem*{lemma*}{Lemma}
\usepackage{newfloat}
\usepackage{listings}
\DeclareCaptionStyle{ruled}{labelfont=normalfont,labelsep=colon,strut=off} 
\floatstyle{ruled}
\newfloat{listing}{tb}{lst}{}
\floatname{listing}{Listing}
\title{No Prior Mask: Eliminate Redundant Action for Deep Reinforcement Learning}
\author{
    Dianyu Zhong\equalcontrib,
    Yiqin Yang\equalcontrib,
    Qianchuan Zhao\thanks{ Advising.}
}
\affiliations{
    Department of Automation, Tsinghua University\\

    \{zhongdy19,  yangyiqi19\}@mails.tsinghua.edu.cn, zhaoqc@tsinghua.edu.cn
}

\usepackage{bibentry}

\begin{document}

\maketitle
\begin{abstract}
The large action space is one fundamental obstacle to deploying Reinforcement Learning methods in the real world. 
The numerous redundant actions will cause the agents to make repeated or invalid attempts, even leading to task failure. 
Although current algorithms conduct some initial explorations for this issue, they either suffer from rule-based systems or depend on 
expert demonstrations,
which significantly limits their applicability in many real-world settings. 
In this work, we examine the theoretical analysis of what action can be eliminated in policy optimization and propose a novel redundant action filtering mechanism. 
Unlike other works, our method constructs the similarity factor by estimating the distance between the state distributions, which requires no prior knowledge. 
In addition, we combine the modified inverse model to avoid extensive computation in high-dimensional state space.
We reveal the underlying structure of action spaces and propose a simple yet efficient redundant action filtering mechanism named No Prior Mask~(NPM) based on the above techniques.
We show the superior performance of our method by conducting extensive experiments on high-dimensional, pixel-input, and stochastic problems with various action redundancy tasks.
Our code is public online at https://github.com/zhongdy15/npm.
\end{abstract}

\section{Introduction}
Although Reinforcement Learning~(RL) methods have found tremendous success in challenging domains, including video games~\citep{mnih2015human} and classic games~\cite{silver2016mastering, yang2021believe, ma2021offline}, deploying them in real-world applications is still limited. 
One of the main challenges towards that goal is dealing with large action spaces.
For example, maximizing long-term portfolio value using various trading strategies in the finance system~\cite{jiang2017deep}, regulating the voltage level of all the units in the power system~\cite{glavic2017reinforcement}, and generating recommendation sequences from an extensive collection of options~\cite{sidney2005integrating} pose significant challenges due to the numerous actions involved.
The redundant actions in these tasks will cause the agents to make repeated invalid attempts, even leading to task failure.



To solve this issue, current researchers attempt to utilize the underlying structure in the action space to filter the irrelevant actions.
One paradigm for algorithm design incorporates the prior knowledge into the learning algorithm.
There are several works in the existing empirical literature, especially in challenging domains such as Dota2~\cite{berner2019dota}, 
StarCraftII~\cite{vinyals2017starcraft} and Honor of Kings~\cite{ye2020mastering}. 
These approaches build simple handcrafted action masks to determine the availability of an action, such as checking if a valid target is nearby or if the ability is on cooldown.
~\citet{ye2020mastering} demonstrates that action mask can largely reduce the training time and ~\citet{huang2021gym} conducts ablation studies on action mask, revealing its substantial potential in enhancing performance.
However, despite the interpretability and efficiency of these methods, it remains exceptionally challenging to manually construct an accurate relationship between actions, particularly in high-dimensional and state-dependent scenarios with limited expert information.


On the other hand, recent advances were made in autonomously learning such underlying structures in the action space.
For example, ~\citet{dulac2015deep} leverages prior information about the actions to embed them in a continuous space.
~\citet{NEURIPS2018_645098b0} utilizes an external elimination signal provided by the environment to predict invalid actions.
~\citet{Tennenholtz2019TheNL} combines existing methods in natural language processing (NLP) to
group
similar actions and utilizing relations between
different actions from expert data.
~\citet{chandak2019learning} provides an algorithm to learn and use action representations on large-scale real-world problems.
~\citet{baram2021action} proposes maximizing the next states' entropy instead of action entropy to minimize action redundancy.

Nevertheless, the aforementioned methods either rely heavily on prior information or are tightly coupling with the policy, which significantly limits their applicability in various real-world scenarios.
Therefore, the shortcoming of the current redundant action filtering mechanism naturally leads to the following question:

\begin{center}
    {{\it How to learn the underlying structure accurately in the action space without prior knowledge?}}
\end{center}


To answer this question, we start with the theoretical analysis of what action can be eliminated in policy optimization.
Intuitively, actions with the same state transition should be classified into similar actions.
Based on this analysis, we propose a redundant action classifying mechanism, which constructs the similarity factor by estimating the distance between the state distributions.
In addition, we combine the modified inverse model to avoid extensive computation in the high-dimensional state space, which makes our algorithm simple yet efficient.
We conduct extensive experiments on various domains, including synthetic action redundancy, combined action redundancy and state-dependent action redundancy.
The experimental results show that our method performs better than other baselines.

\subsection{Important Contributions:}
\begin{enumerate}
    \item We propose a mask-based reinforcement learning framework without any prior knowledge, which can scale up to high-dimensional pixel-based observations.
    To the best of our knowledge, our work is the first study of constructing no prior state-dependent action mask, leading our method to perform better than other methods.
    \item We give a novel theoretical analysis revealing what kind of mask is reasonable and feasible. Our work stands out as the first attempt to bridge the gap between the practical technique of action mask and the theoretical foundations of Markov Decision Processes (MDPs).
    \item The mask learning phase is reward-free, indicating our mask mechanism is unrelated to policy and easily transfers across multi-tasks.
    We validate our method's effectiveness, simplicity, and transfer ability on both single-tasks and multi-tasks.
\end{enumerate}

\section{Related work}
\paragraph{RL in Large Action Space}
There are extensive works to apply RL algorithms to environments with large action spaces, which can be roughly divided into using prior knowledge and without prior knowledge.
On the one hand, researchers find that better performances can be achieved when actions are factored into their primary categories~\cite{sharma2017learning} according to the underlying compositional structure in Atari 2600~\cite{bellemare2013arcade}.
\citet{dulac2015deep} attempt to embed actions with prior information in a continuous space and find the closest neighbors of promising ones.
\citet{zahavy2018learn} directly learn about redundant or irrelevant actions with external elimination signals provided by the environment and eliminate them from being sampled in text-based games.
\citet{Tennenholtz2019TheNL} adopts the Negative-Sampling procedure with expert demonstrations to gain a deeper understanding of actions.
However, valuable prior information is scarce and expensive, which limits the application of these methods.


Unlike prior-based methods, learning in large action spaces without prior knowledge is hard.
\citet{chandak2019learning} showed how to learn and use representations without prior knowledge in which an embedding is used as part of the policy's structure to train an agent.
\citet{baram2021action} explored to incorporate maximum entropy between distributions over the next states and penalized it in reward as action redundancy term.
However, these methods have a tight coupling with the policy, limiting the optimized policy's generalization performance across multi-tasks.

Meanwhile, there are several works on discretizing continuous actions and learning the structure within the action space. ~\citet{metz2017discrete} introduces a creative approach to discretizing high-dimensional continuous actions through the sequential combination of one-dimensional discrete actions. ~\citet{tavakoli2018action} explores the intriguing concept of utilizing hypergraphs to combine actions within the action space. ~\citet{tavakoli2021learning} presents an innovative neural architecture that capitalizes on a shared decision module and distinct network branches for each action dimension. ~\citet{seyde2022solving} offers a new lens on Q-learning by amalgamating action discretization with value decomposition. By framing single-agent control as a cooperative multi-agent RL problem, this paper contributes a fresh perspective to addressing continuous control challenges. ~\citet{dadashi2021continuous} introduces an elegant approach by discretizing continuous actions based on prior knowledge and subsequently learning exclusively within the corresponding discrete action space during online training.

\paragraph{Action Mask} Action mask has achieved wide attention in dealing with numerous invalid actions tasks, such as Dota 2~\cite{berner2019dota}, Starcraft II~\cite{vinyals2017starcraft} and Honor of Kings~\cite{ye2020mastering}.
These methods build simple action filters manually, determining whether the action is available.
At each time step, the algorithm restricts available actions from the full action space based on the constructed filters and optimizes policy in a sub-action space.
~\citet{ye2020mastering} demonstrates that action mask can largely reduce the training time from 80 hours to 65 hours. ~\citet{huang2021gym} reveals that action mask considerably improves performance on real-time strategy (RTS) games.
However, learning a valid action mask without prior knowledge and promoting algorithm performance is not explored well.

\begin{figure*}[t]
    \centering
    \includegraphics[width=0.8\textwidth]{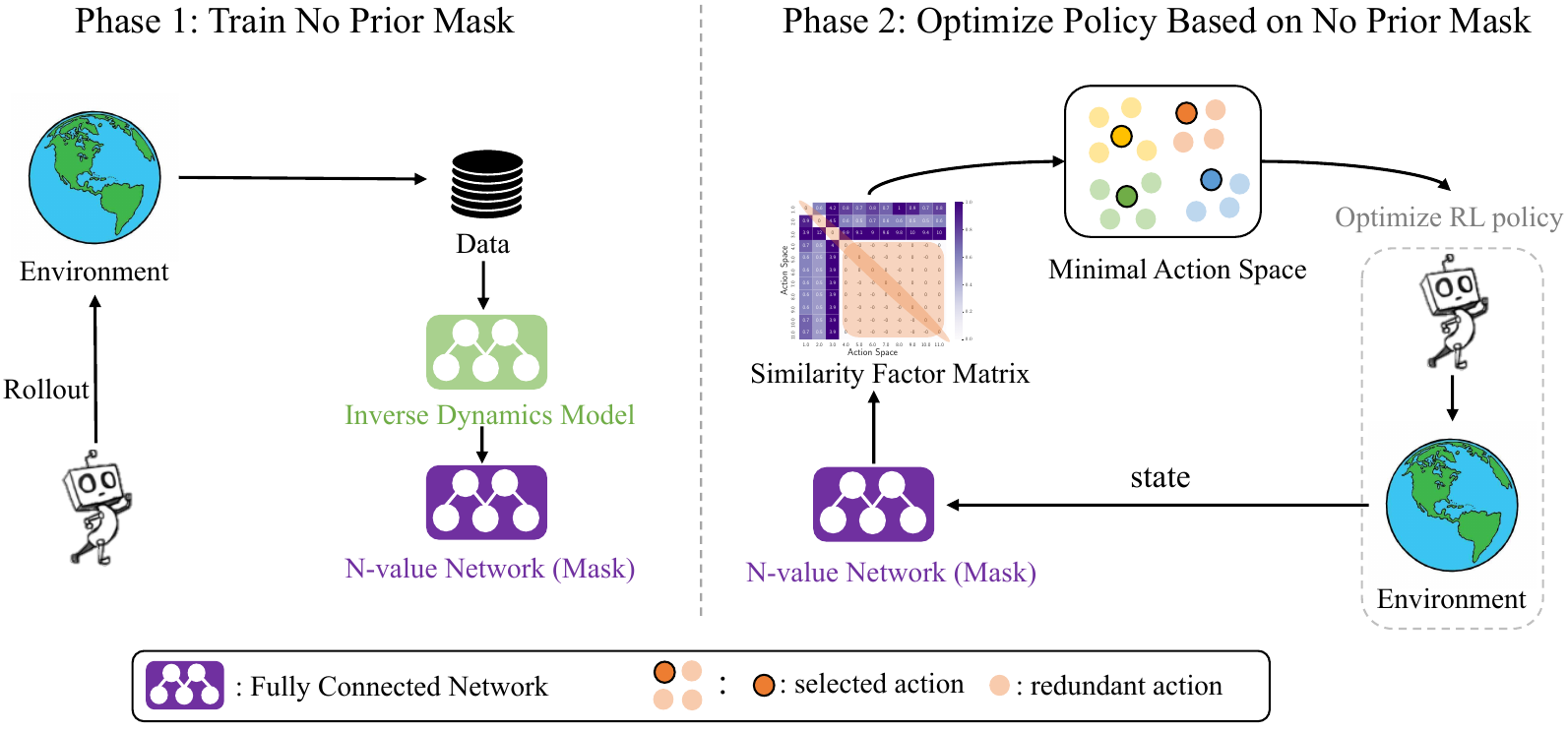}
    \caption{The framework of NPM.
    Our method consists of two phases: The first is to train the no prior mask.
    The second phase uses the trained mask to calculate the similarity factor matrix and eliminate the redundant actions.}
    \label{fig: framework}
\end{figure*}

\paragraph{MDP Homomorphisms and Bisimulation} There are several works in MDP homomorphisms and action abstractions. ~\citet{taylor2008bounding,ravindran2004algebraic} define a behavior similarity metric for measuring state-action pair similarity and obtains error bounds. ~\citet{panangaden2023policy} extend the definition of MDP homomorphisms to the setting of continuous state and action spaces and derive a policy gradient theorem on the abstract MDP. ~\citet{thomas2017independently} learn controllable factors along with policies. ~\citet{sallans2004reinforcement} use an undirected graphical model to approximate the value function with factored states and actions.
However, the methods utilizing the Hausdorff metric ~\cite{taylor2008bounding} or the Kantorovich metric~\cite{panangaden2023policy} lead to unaffordable computational complexity in high-dimensional states, which limits the applicability to broad scenarios.

\paragraph{Action Priors} There are several works in accelerating learning by reducing the number of actions in RL. ~\citet{sherstov2005improving,fernandez2006probabilistic,rosman2012good,biza2021action} utilize a set of different optimal policies from similar tasks to construct a small fraction of the domain’s action set or bias exploration away from sub-optimal actions. ~\citet{rosman2015action} define action priors over observation features rather than states, and ~\citet{abel2015goal} model the optimal actions through experience to avoid searching for sub-optimal actions.~\citet{yan2022ceip} combine multiple implicit priors with task-specific demonstrations to outperform in challenging environments.
However, the research on Action Priors focus on reducing sub-optimal actions instead of redundant actions discussed in our work.

\section{Background}
We consider the Markov Decision Processes~(MDPs), defined by the tuple $(\mathcal{S}, \mathcal{A}, \mathcal{P}, r, \gamma)$, where $\mathcal{S}$ is a state space, $\mathcal{A}$ is an action space, $\gamma\in [0,1)$ is the discount factor and $\mathcal{P}: \mathcal{S}\times\mathcal{A} \rightarrow \Delta(\mathcal{S}), r:\mathcal{S}\rightarrow [0, r_{\text{max}}]$ are the transition function and reward function, respectively.
In this work, we restrict our focus to MDPs with finite action sets, and $|\mathcal{A}|$ denotes the size of the action set.
We also assume a fixed distribution $\mu_0 \in \Delta(\mathcal{S})$ as the initial state distribution.
The goal of an RL agent is to learn a policy $\pi: \mathcal{S} \rightarrow \Delta{(\mathcal{A})}$
under dataset $\mathcal{D}$, which maximizes the expectation of a discounted cumulative reward: $J(\pi)=\mathbb{E}_{\mu_0,\pi}\left[\sum_{t=0}^{\infty}\gamma^t r(s_t)\right]$.

For any policy $\pi$, the corresponding state-action value function is $Q^{\pi}(s,a)=\mathbb{E}[\sum_{k=0}^{\infty}\gamma^k r_{t+k}|S_t=s,A_t=a,\pi]$.



\subsection{Action Mask Mechanism}
The action mask mechanism significantly reduces the complexity of exploration for some complicated cases, such as real-time strategy games or job scheduling.
In standard value iteration methods, the agent will make decisions by maximizing $Q(s, a)$ among the full action space $\mathcal{A}$:

$$a = {\arg\max}_{a\in\mathcal{A}}Q(s,a).$$

However, if $|\mathcal{A}|$ is too large and contains numerous redundant actions, the difficulty of the exploration is greatly increased.
For this reason, we define the invalid action space $\mathcal{A}^{\text{invalid}}\subseteq \mathcal{A}$ to reduce the difficulty of the exploration, which can be state-dependent or time-varying.
Given $\mathcal{A}^{\text{invalid}}$, value-based RL methods will replace the full action space by the valid action space $\mathcal{A} \backslash \mathcal{A}^{\text{invalid}}$:

\begin{equation}
\label{eq: value-based}
a = {\arg\max}_{a \in \mathcal{A} \backslash \mathcal{A}^{\text{invalid}}} Q(s,a).
\end{equation}

As for policy gradient methods, the action mask mechanism will forcibly set the sampling probability of the redundant actions to zero and re-normalize:
\begin{equation}
    \label{eq: pi mask}
    \pi_{\text{mask}}(a|s)=
    \left\{ 
    \begin{array}{cc}
     \frac{\pi(a|s)}{1-\sum_{a \in \mathcal{A}^{\text{invalid}}}\pi(a|s)} &, a \notin \mathcal{A}^{\text{invalid}} \\
     0 &, a \in \mathcal{A}^{\text{invalid}}. 
    \end{array}
    \right.
\end{equation}




\section{Method}
We start with the theoretical analysis of what action can be eliminated in policy optimization.
Based on this analysis, we propose a redundant action classifying mechanism, which constructs the similarity factor by estimating the distance between the state distributions.
Further, we introduce the modified inverse model to avoid extensive computation.
Finally, we show the overall framework of our algorithm in Figure~\ref{fig: framework} and Algorithm~\ref{alg: NPM}.


\subsection{Identify Redundant Actions based on the Similarity Factor}\label{sec: method1}

{\it What actions can be eliminated in policy optimization?}
At first sight, actions transferring from the same initial state $s_t$ to the same next state $s_{t+1}$ should be classified into similar actions since they have the same effect on the state transition.
We formulate this intuition by introducing the similarity factor as follows:


\begin{definition}[Similarity Factor]
    \label{def: similarity factor}
    We define the similarity factor $M(s_t, a_i, a_j)$ by KL-divergence following state $s_t$ and any two actions $a_i$, $a_j$:
    \begin{align}
        M(s_t,a_i,a_j) = D_{\text{\rm KL}}(P(s_{t+1}|s_t,a_i) \| P(s_{t+1}|s_t,a_j)).
    \end{align}
\end{definition}



We can use the similarity factor to identify whether the actions are redundant or irrelevant based on the definition.
Specifically, if the value of $M(s_t, a_i, a_j)$ is small, the corresponding two actions have similar effects on the state $s_t$ and should be aligned into the same cluster.
Therefore, we have the following formulation:



\begin{definition}[State-dependent Action Clusters]
    We define the state-dependent action cluster $A_{s,\epsilon}^m$ over state $s$ based on the similarity factor:
    \begin{align}
        &\mathcal{A} = \bigcup_{m \in \{0,1,...,l-1\}} A_{s,\epsilon}^m, \\
        &\text{s.t.} \quad \forall a_i, a_j \in A_{s,\epsilon}^m, \quad M(s,a_i,a_j) < \epsilon
    \end{align}
\end{definition}

where $A_{s,\epsilon}^m$ is the sub-action space and $A_{s,\epsilon}^i \cap A_{s,\epsilon}^j = \emptyset$ if $i \neq j$.
In practice, we choose arbitrarily one action $a_{A_m} \in A_{s,\epsilon}^m$ (e.g. the first action) to represent the sub-action space $A_{s,\epsilon}^m$ and mask out all the redundant actions.
Then, based on the selected action $a_{A_m}$, we can reduce the original action space $\mathcal{A}$ in the state $s$ into the minimal action space $\mathcal{A}_{s}^{\epsilon} = \{ a_{A_0}, a_{A_1},...,a_{A_{l-1}}\}$.
Next, similar to Equation~\ref{eq: value-based} and Equation~\ref{eq: pi mask}, the policy gradient methods and value-based methods can be optimized by simply replacing the full action space with the minimal action space $\mathcal{A}_{s}^{\epsilon}$.


Actions in the same clusters have similar effects on the state. Intuitively, eliminating the redundant actions will have little effect on the policy's performance. Based on the above motivation, we have the following theoretical analysis:



\begin{lemma}

\label{lemma1}
For any policy $\pi$ and $A_{s,\epsilon}^m$, we can choose arbitrarily one action $a_{A_m} \in A_{s,\epsilon}^m$ to represent the sub-action space. We denote
\begin{align}
\pi_{A_{s,\epsilon}^m}(a \mid s)= \begin{cases}\sum_{a^{\prime} \in A_{s,\epsilon}^m} \pi\left(a^{\prime} \mid s\right) & , a = a_{A_m} \\ 0 &, a \in A_{s,\epsilon}^m \backslash \{ a_{A_m}\} \\ \pi(a \mid s) & , \text {\rm o.w }\end{cases}
\end{align}

then
\begin{align}
\left\|V^\pi-V^{\pi_{A_{s,\epsilon}^m}}\right\|_{\infty} \leq \frac{\gamma r_{max}}{(1-\gamma)^2} \sqrt{2\epsilon}
\end{align}
\end{lemma}

Lemma~\ref{lemma1} shows that eliminating the redundant actions based on the similarity factor has little effect on the policy's performance.
Please refer to Appendix~\ref{appendix: proof} for the proof.







\subsection{Estimate Similarity Factor using Modified Inverse Dynamic Model}
\label{subsec: inverse model}

Although we can identify the redundant actions based on the similarity factor, accurately calculating $M(s_t, a_i, a_j)$ is extremely hard since it requires the distance on the state distribution.
Therefore, the key issue is reducing the computational complexity of $M(s_t, a_i, a_j)$. 
We derive a form equivalent to Definition~\ref{def: similarity factor}, shown in Lemma ~\ref{lemma2} to solve this issue.
Compared with Definition~\ref{def: similarity factor}, Lemma~\ref{lemma2} transfers the computation on the state space to the action space, making calculating $M(s_t, a_i, a_j)$ feasible.


\begin{lemma}
\label{lemma2}
The similarity factor $M$ can be divided by two terms with the same form $N$, which we refer to as N-value network: 
\begin{align*}
M(s_t,a_i,a_j) = N(s_t,a_i,a_i) - N(s_t,a_i,a_j)
\end{align*}
where
\begin{align*}
N(s_t,a_i,a_j) = \mathbb{E}_{s_{t+1} \sim P(\cdot|s_t,a_i)}\log \left[ \frac{P^{\pi}(a_j|s_t,s_{t+1})}{\pi(a_j|s_t)} \right].
\end{align*}
For $i = j$, then we have the former term
\begin{align*}
N(s_t,a_i,a_i) = \mathbb{E}_{s_{t+1} \sim P(\cdot|s_t,a_i)}\log \left[ \frac{P^{\pi}(a_i|s_t,s_{t+1})}{\pi(a_i|s_t)} \right].
\end{align*}
\end{lemma}

In Lemma~\ref{lemma2}, we introduce the inverse dynamic model $P^{\pi}(a_t|s_t,s_{t+1})$ to calculate $M(s_t, a_i, a_j)$, which greatly reduce the computational complexity of our method. 
Please refer to Appendix~\ref{appendix: proof} for the proof.

We note that $P^{\pi}(a_t|s_t,s_{t+1})$ is different from $P(a_t|s_t,s_{t+1})$, which ignores the $\pi$ and do not take the policy distribution over $s_t$ into account.
The conflict induced by the $P(a_t|s_t,s_{t+1})$ will be severe when training on the data collected by varying policies.


When the policy $\pi$ is fixed, $P^{\pi}(a_t | s_t, s_{t+1})$ is equivalent to $P(a_t | s_t, s_{t+1})$. However, in more complex environments, where the agent needs to explore a vast state space, a fixed policy may not lead to effective exploration.  
To solve this issue, we have to adopt multiple exploration policies~(e.g., curiosity reward~\cite{pathak2017curiosity}).
With the collected data coming from gradually varying policies, the modified inverse dynamics model lead to better performance in Experiment Section.

Therefore, we construct the modified inverse dynamic model $P^{\rm inv}_{\psi}(a_t|s_t,s_{t+1},\pi(\cdot|s_t))$ with parameter $\psi$ to estimate $ P^{\pi}(a_t|s_t,s_{t+1})$. 
In Experiment Section, we show that the modified inverse dynamic model can converge faster and achieve higher accuracy compared with $P(a_t|s_t,s_{t+1})$ with the collected data coming from gradually varying policies.






\subsection{Practical Implementation}

In this subsection, we show how to train the modified inverse dynamic model $P^{\rm inv}_{\psi}$ and N-value network $N_{\theta}$ parameterized with $\theta$.
First, we train the inverse dynamic model with the following loss:


\begin{align}
\label{eq: mask}
&L(\psi) =\\ \notag &H_{( s_t,a_t,s_{t+1},\pi(\cdot|s_t)) \sim \mathcal{D}}\left[P^{\rm inv}_{\psi}(\cdot|s_t,s_{t+1},\pi(\cdot|s_t)), P^{\rm one-hot}_{a_t}(\cdot)\right],
\end{align}

where $H(\cdot,\cdot)$ is the cross entropy loss and $|P^{\rm one-hot}_{a_t}(\cdot)|=|\mathcal{A}|$.
$P^{\rm one-hot}_{a_t}(\cdot)$ denotes one-hot distribution, where the index value corresponding to $a_t$ is 1, otherwise is 0.


\begin{figure*}[ht]
    \centering\subfloat{\includegraphics[scale=0.24]{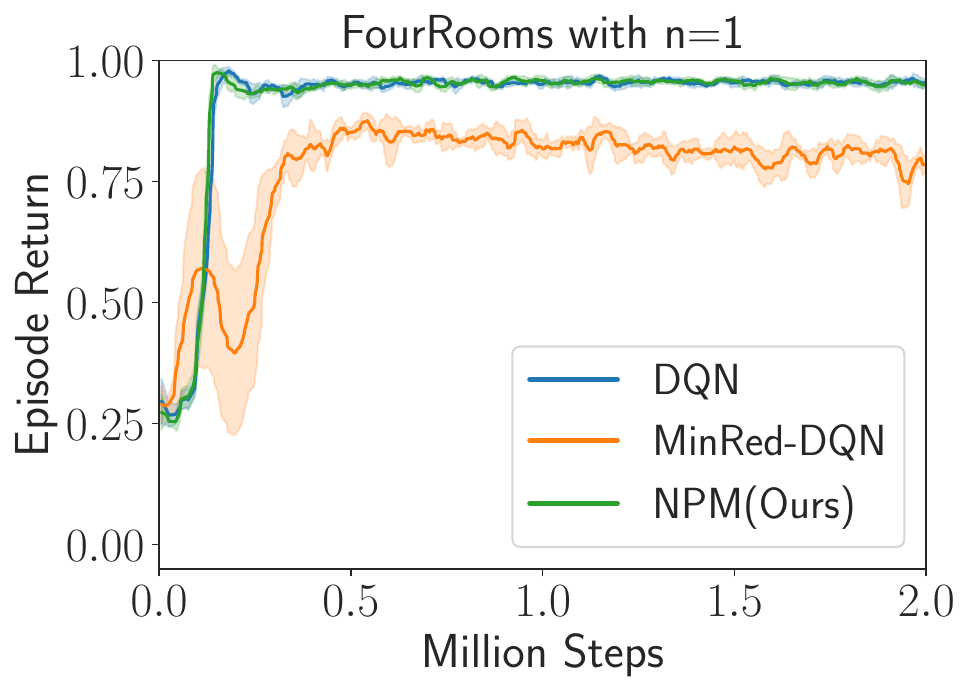}\label{fig: sub_figure1}}
    \hspace{1mm}\subfloat{\includegraphics[scale=0.24]{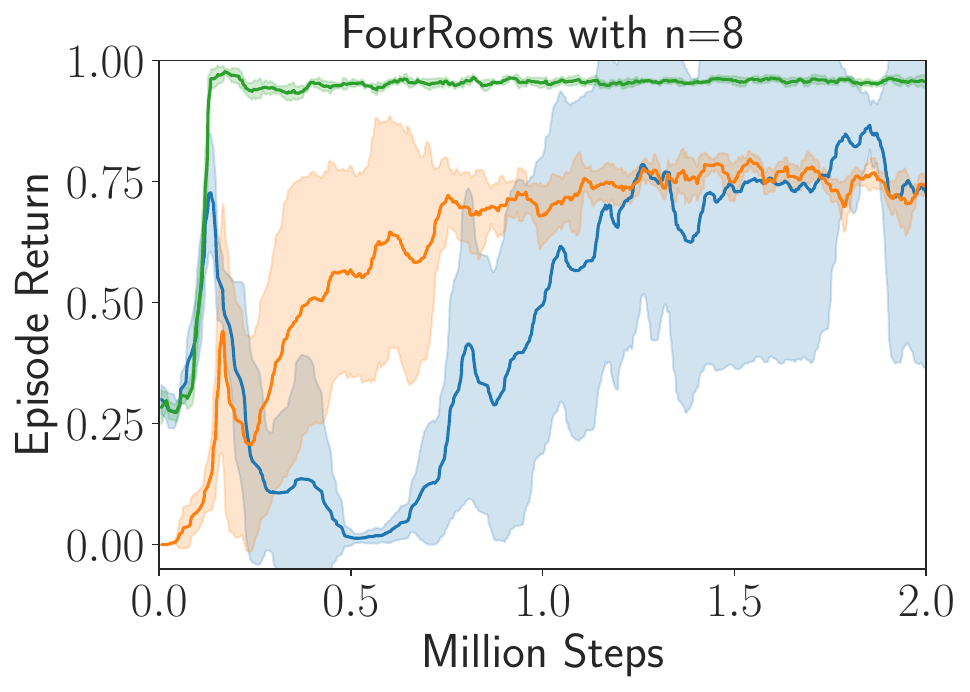}\label{fig: sub_figure4}}
    \hspace{1mm}\subfloat{\includegraphics[scale=0.24]{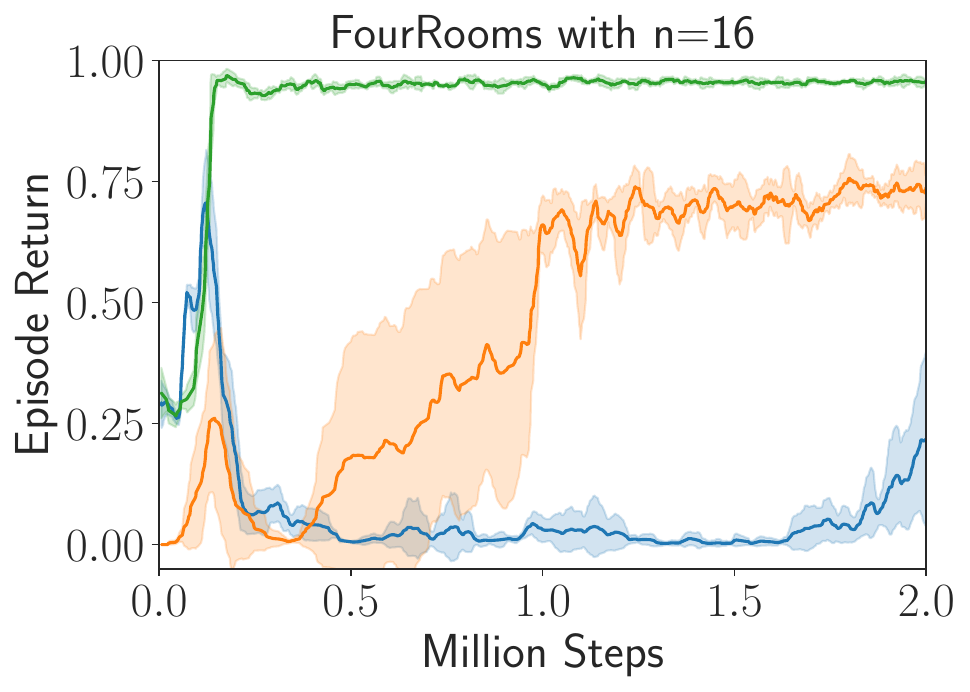}\label{fig: sub_figure5}}
    \hspace{1mm}\subfloat{\includegraphics[scale=0.24]{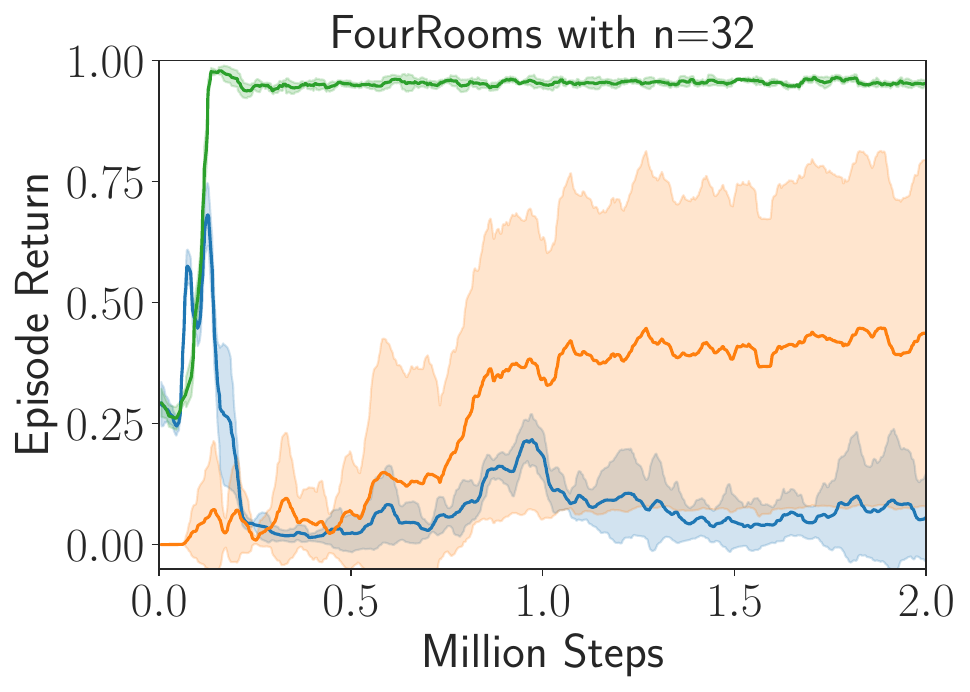}\label{fig: sub_figure6}}

    \subfloat{\includegraphics[scale=0.28]{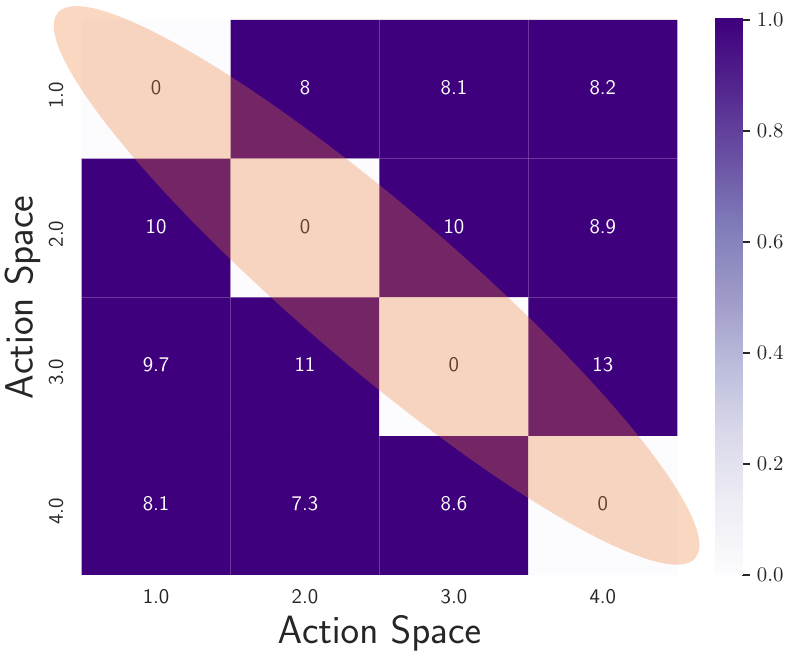}}
    \hspace{1mm}\subfloat{\includegraphics[scale=0.28]{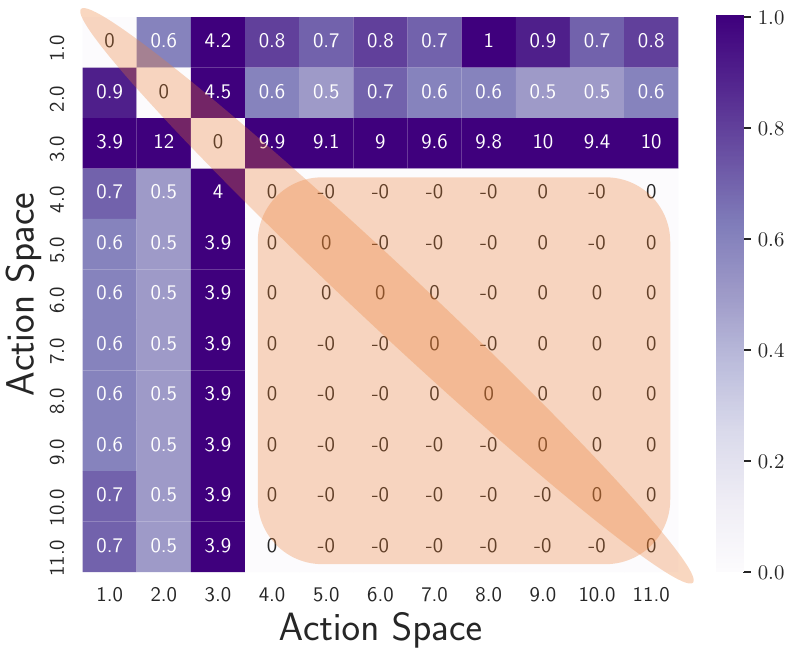}}
    \hspace{1mm}\subfloat{\includegraphics[scale=0.28]{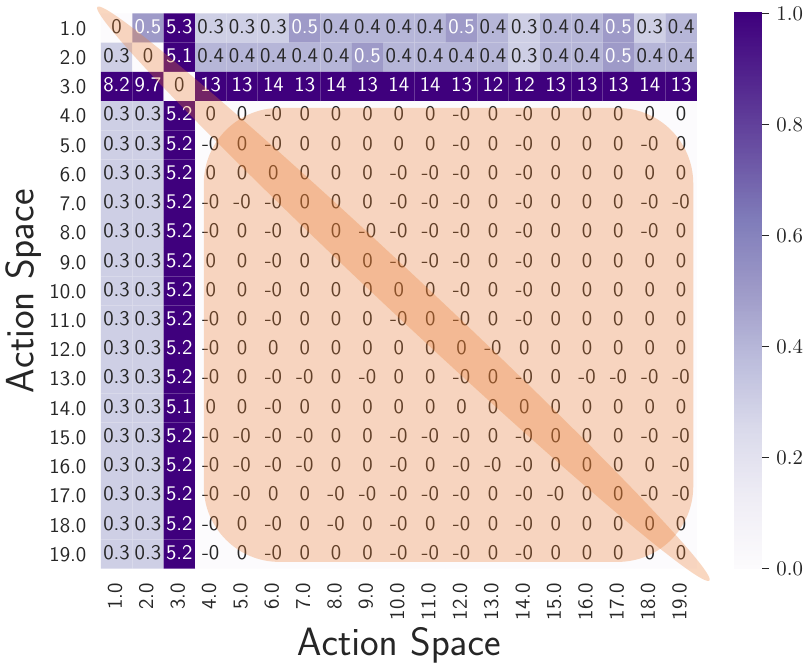}}
    \hspace{1mm}\subfloat{\includegraphics[scale=0.28]{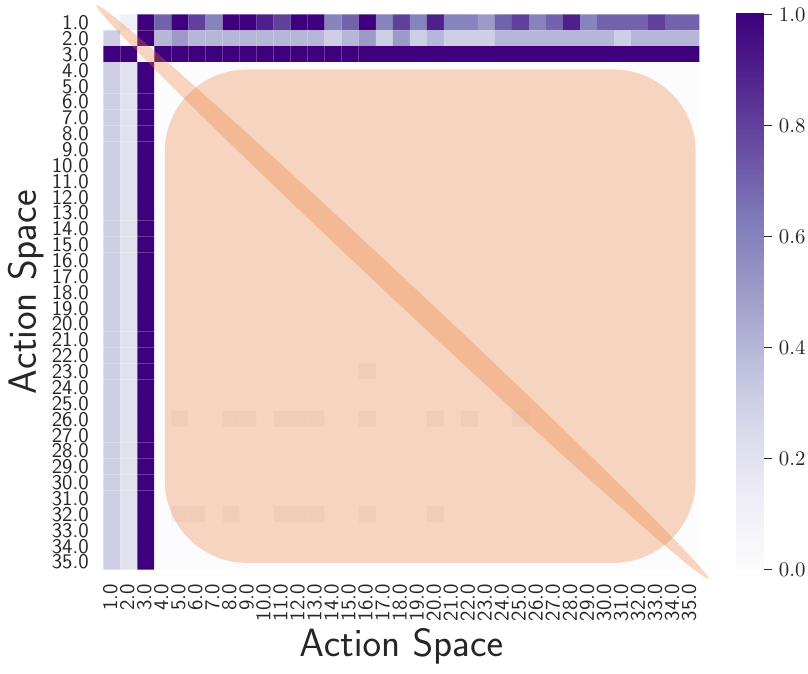}}
    \caption{Upper: The experimental results on FourRooms with synthetic action redundancy $n \in \{1,8,16,32\}$ with five random seeds. 
    Below: Visualization of similarity factor matrix corresponding to each task.}
    \label{fig: FourRoomsexp}
\end{figure*}

As for the N-value network $N_{\theta}$, we adopt the state-action pair $(s,a_t)$ as inputs and output a $|\mathcal{A}|$-dimensional vector.
Actually, there are some alternative plans to construct an N-value network.
For example, we can input only state $s$ and output a similarity factor matrix.
However, with the increase of $|\mathcal{A}|$, the output dimension is positively correlated with ${|\mathcal{A}|}^2$, leading the method difficult to guarantee convergence.
Therefore, based on the Lemma~\ref{lemma2}, we train the N-value network with the following loss:

\begin{align}
\label{eq: N-value}
L(\theta) = &\mathbb{E}_{( s_t,a_t,s_{t+1},\pi(\cdot|s_t)) \sim \mathcal{D}}[ N_{\theta}(s_t,a_t,\cdot) \\ \notag & - \log  \frac{P^{\rm inv}_{\psi}(\cdot|s_t,s_{t+1},\pi(\cdot|s_t))}{\pi(\cdot|s_t)} ]^2
\end{align}




where $P^{\rm inv}_{\psi}(\cdot|s_t,s_{t+1},\pi(\cdot|s_t))$ is the trained inverse dynamic model.
In practice, the modified inverse dynamic model $P^{\rm inv}_{\psi}$ and N-value network $N_{\theta}$ are trained iteratively.
$P^{\rm inv}_{\psi}$ updates faster than N-value network to guarantee the convergence.
Based on the trained N-value network, we calculate the similarity factor between any two actions $a_i, a_j$:
\begin{align}
M(s_t,a_i,a_j) = N_{\theta}(s_t,a_i,a_i) - N_{\theta}(s_t,a_i,a_j).
\end{align}

In the practical implementation, we iterate over each action $a_i$ in the discrete action set to obtain the similarity factor matrix $M(s_t, \cdot, \cdot)$ ($|\mathcal{A}| \times |\mathcal{A}|$), where every element
represents the KL-divergence between any two actions.
Then, we can utilize $M(s_t, \cdot, \cdot)$ to construct state-dependent action clusters and eliminate the redundant actions as illustrated in Lemma~\ref{lemma1}.



We show the overall framework in Algorithm~\ref{alg: NPM} in Appendix~\ref{appendix: alg} and Figure~\ref{fig: framework}.
Our algorithm consists of two parts.
The former is to train the similarity factor model.
The latter is eliminating redundant actions during policy optimization.
We note that in complex tasks we instead adopt the curiosity reward~(e.g., $r(s,a)=N(s,a)^{-\frac{1}{2}}$)~\cite{pathak2017curiosity} to update the initial random policy while training the similarity factor model.





\section{Experiments}
\label{sec: exp}

In this section, we aim to evaluate the effectiveness of the redundant action filtering mechanism and answer the following questions: 
(1) How does NPM perform compared with other methods in various action redundancy tasks?
(2) Why and When is NPM effective?
(3) Whether the learned mask can be transferred across multi-tasks?


\subsection{Experimental Setting}
We evaluate our method on three action redundancy tasks: synthetic action redundancy, combined action redundancy, and state-dependent action redundancy including Minigrid environments~\cite{minigrid,chevalier2018babyai} and Atari benchmark.
In all experiments, we utilize image as the sole input and provide evidence for NPM's scalability to high-dimensional pixel-based observations.



\begin{figure}
    \centering\subfloat{\includegraphics[scale=0.06]{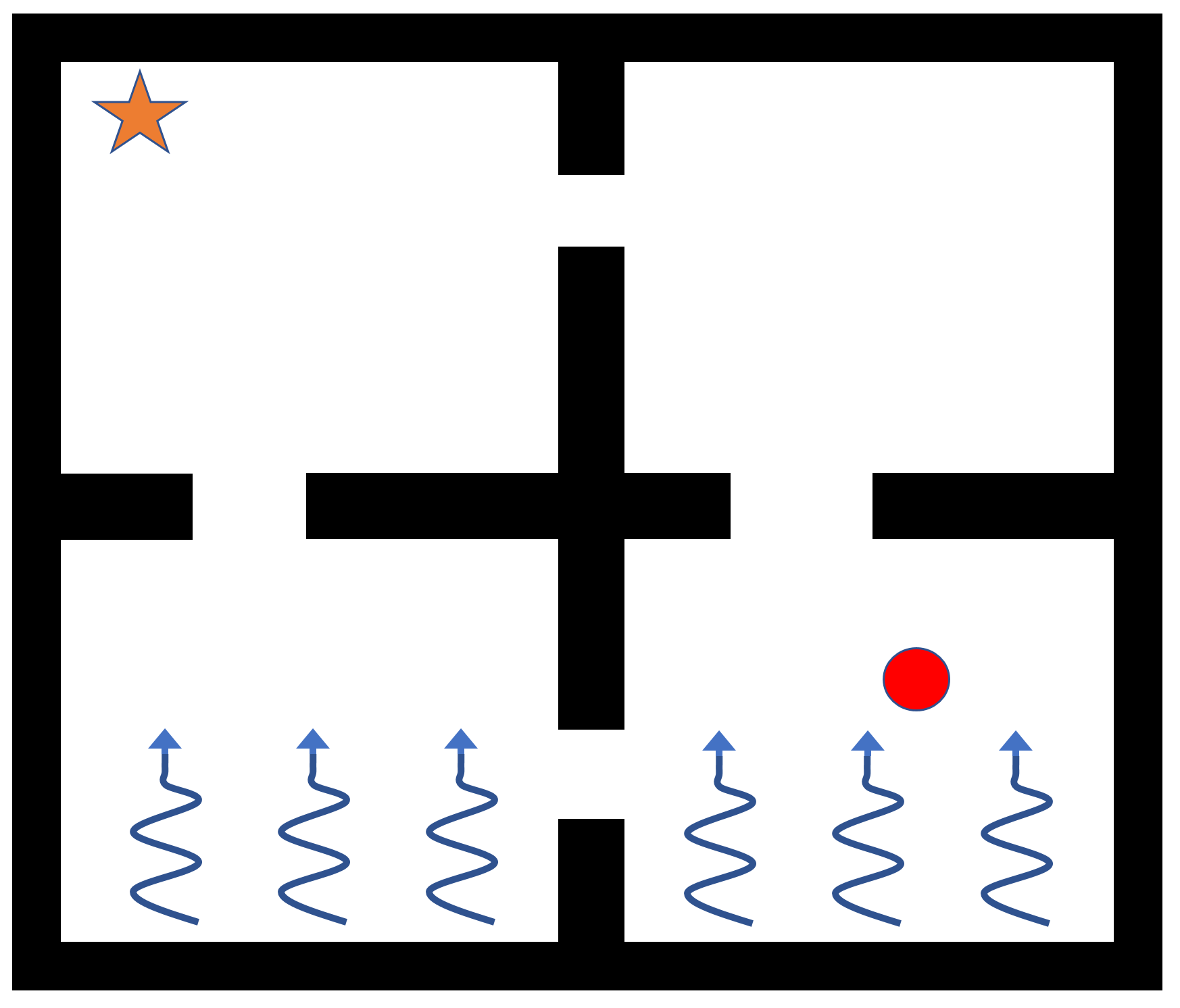}}
    \hspace{0.5cm}\subfloat{\includegraphics[scale=0.048]{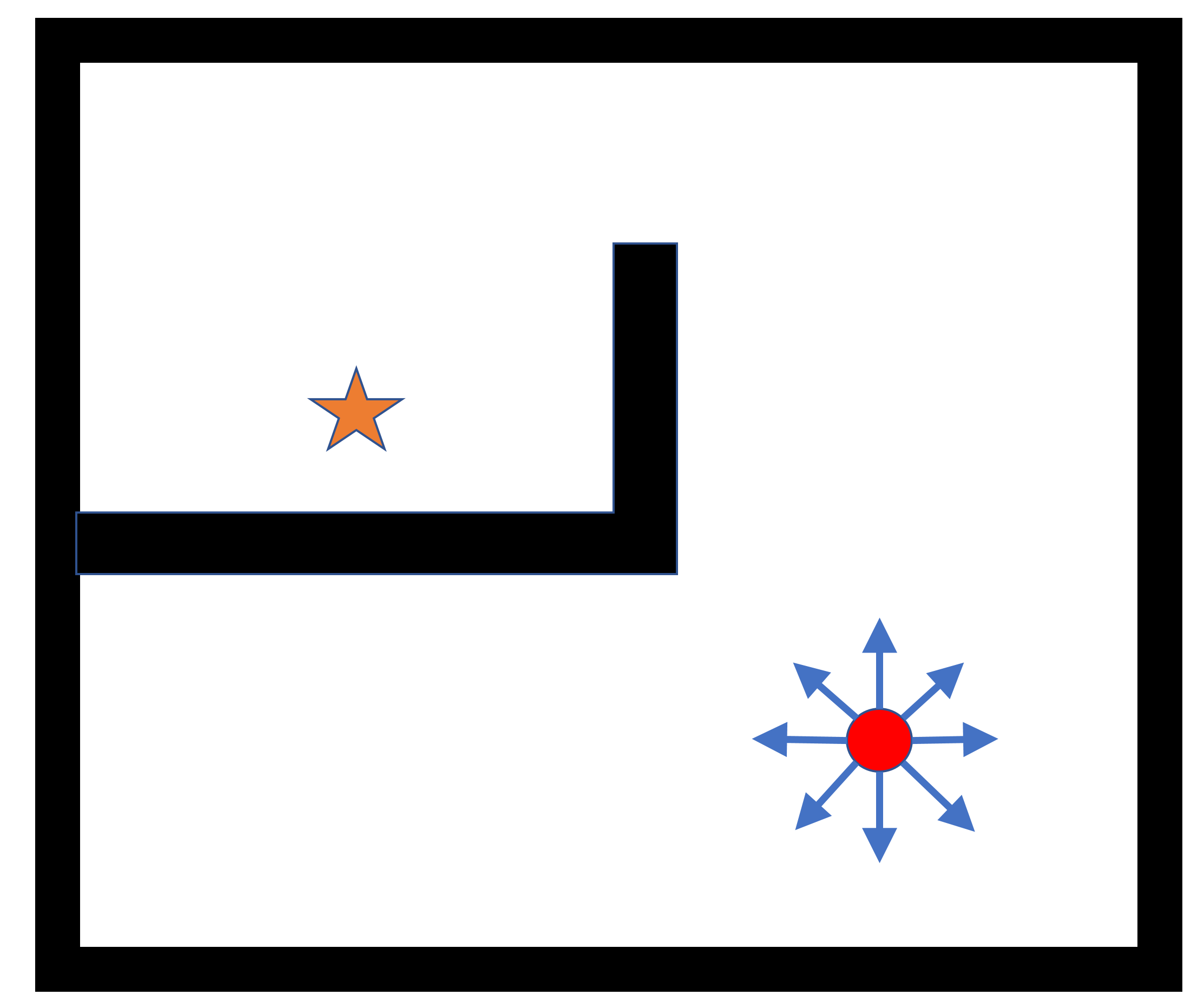}}
    \caption{Left: Four room task. Right: Maze task.}
    \vspace{-0.5cm}
    \label{fig: task description}
\end{figure}

\paragraph{Synthetic Action Redundancy.}
We consider a controlled synthetic four-room task consistent with~\cite{baram2021action}, shown in Figure~\ref{fig: task description}.
In this task, agent~(red circle) must achieve the goal~(star) to earn the reward within 100 steps.
Meanwhile, the wind in the vertical direction will interfere agent with probability $p$.
We construct a synthetic redundant action space $\mathcal{A}=\{\textit{\rm Top, Bottom, Left,} \overbrace{\rm Right, ..., Right}^{n}\}$, for which the action $\textit{\rm Right}$ was repeated $n$ times.
To reach the goal, the agent must counteract redundancy in action space.




\begin{figure*}[t]
    \centering
    \subfloat{\includegraphics[scale=0.22]{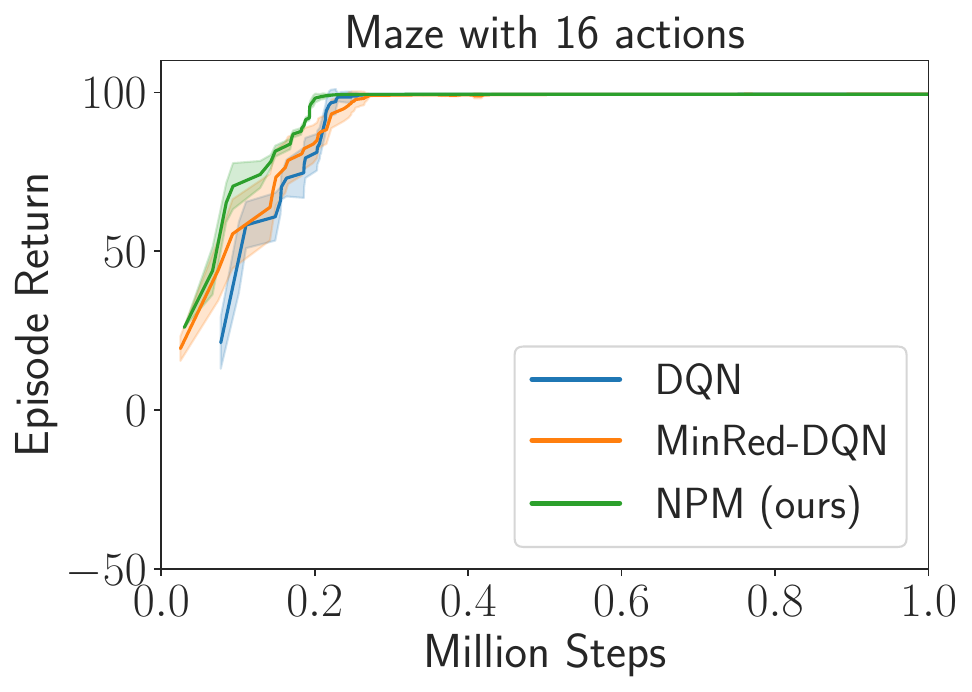}\label{fig: sub_figure1}}
    \hspace{1mm}
    \subfloat{\includegraphics[scale=0.22]{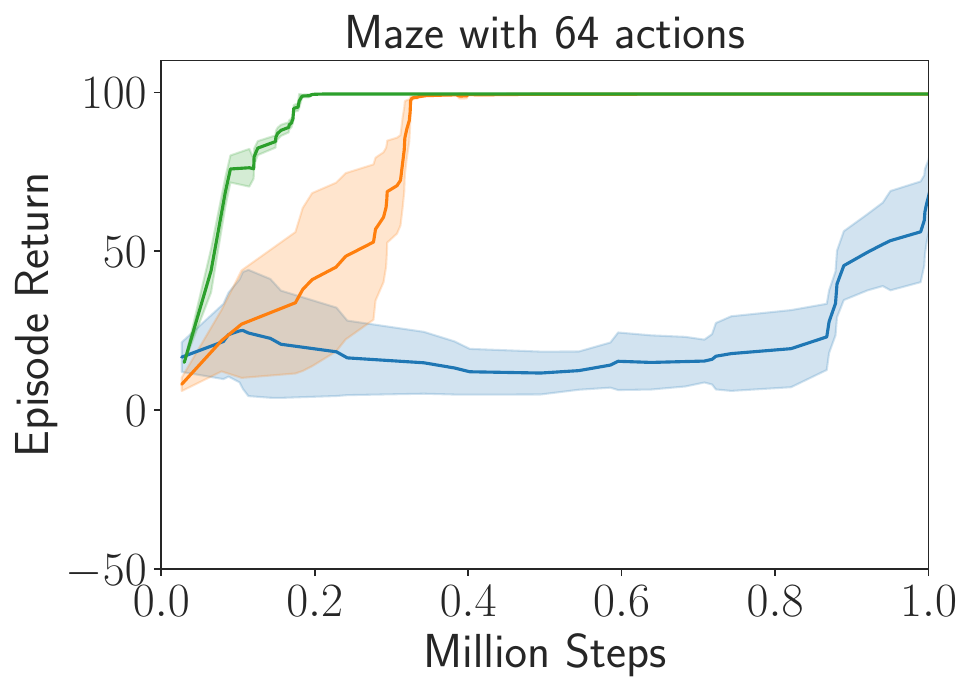}\label{fig: sub_figure2}}
    \hspace{1mm}
    \subfloat{\includegraphics[scale=0.22]{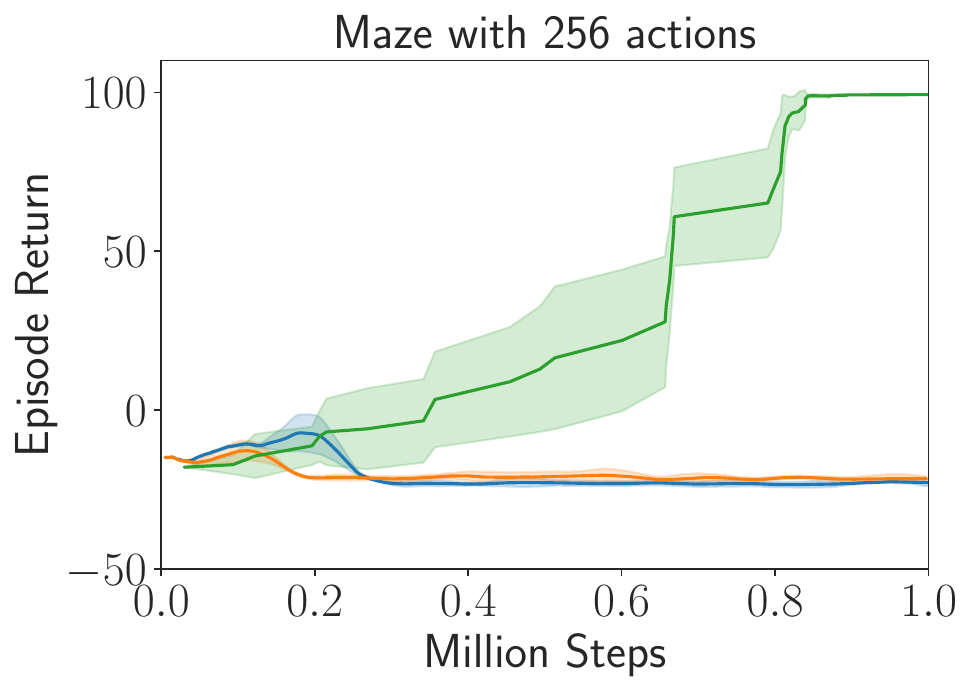}\label{fig: sub_figure3}}
    \caption{The experimental results on 
    Maze with combined action redundancy with actions = $\{16,64,256\}$.
    The experiments are conducted under five random seeds.}
    \label{fig: mazeexp}
\end{figure*}

\begin{figure*}[t]
    \centering
    \subfloat{\includegraphics[scale=0.20]{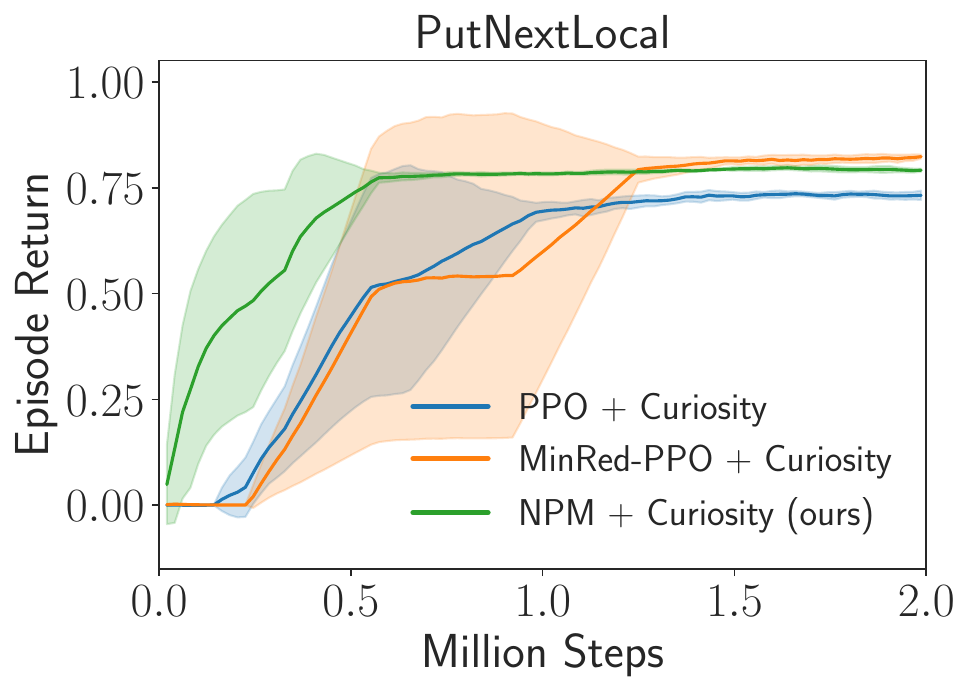}\label{fig: sub_figure1}}
    \subfloat{\includegraphics[scale=0.2]{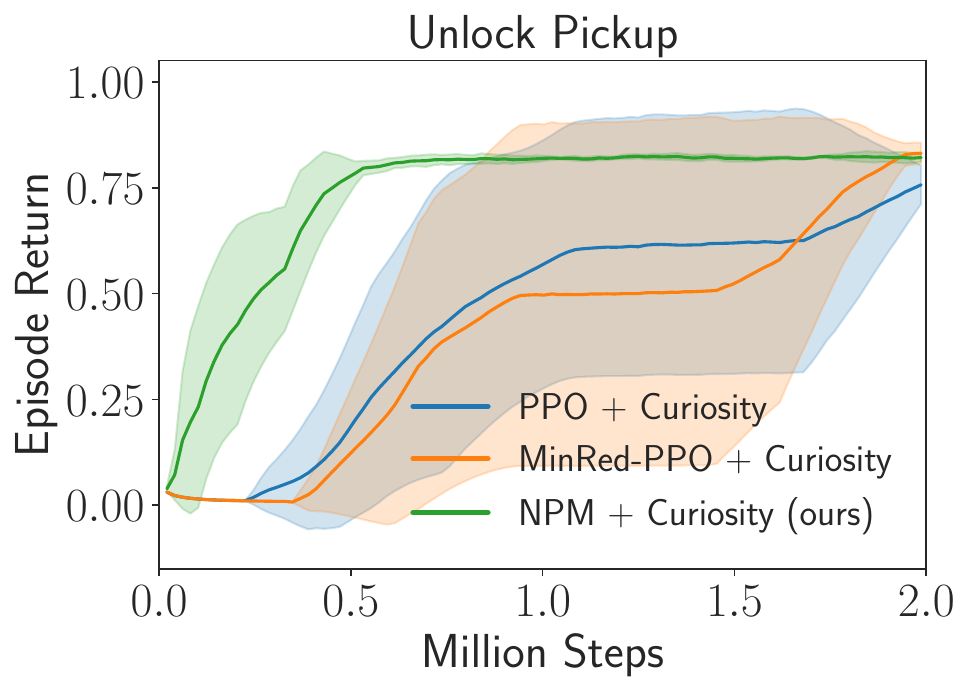}\label{fig: sub_figure1}}
    \subfloat{\includegraphics[scale=0.20]{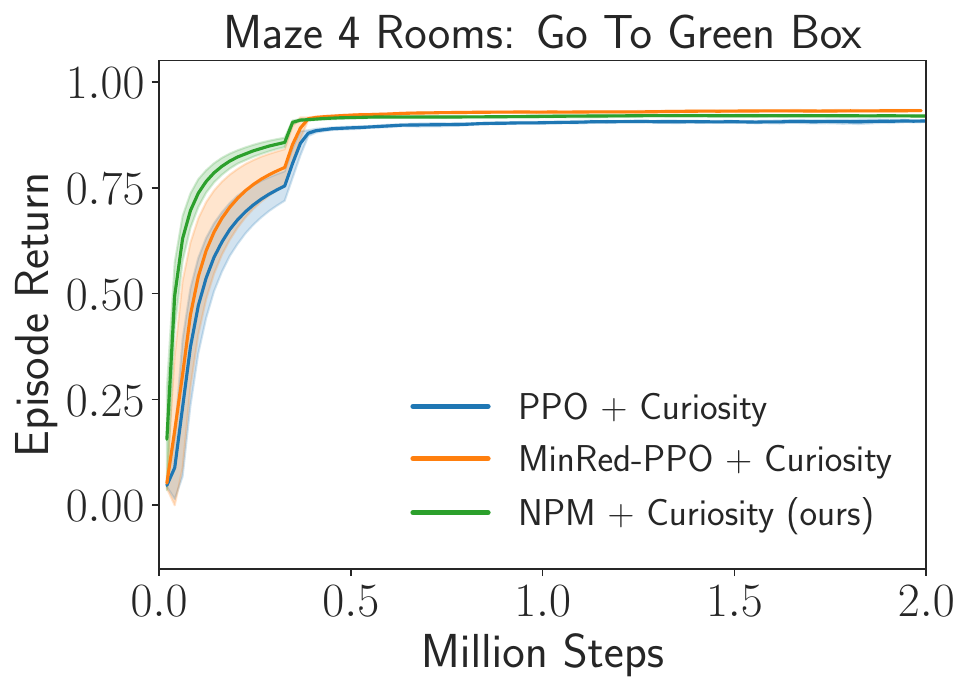}}\label{fig: sub_figure1}
    \subfloat{\includegraphics[scale=0.2]{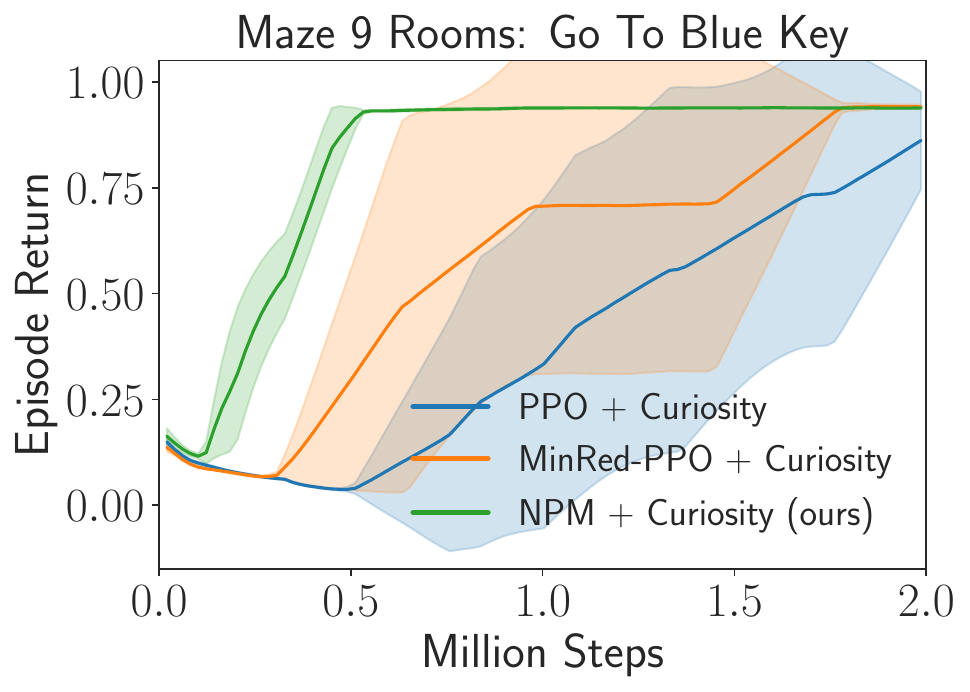}\label{fig: sub_figure2}}

    \subfloat{\includegraphics[scale=0.20]{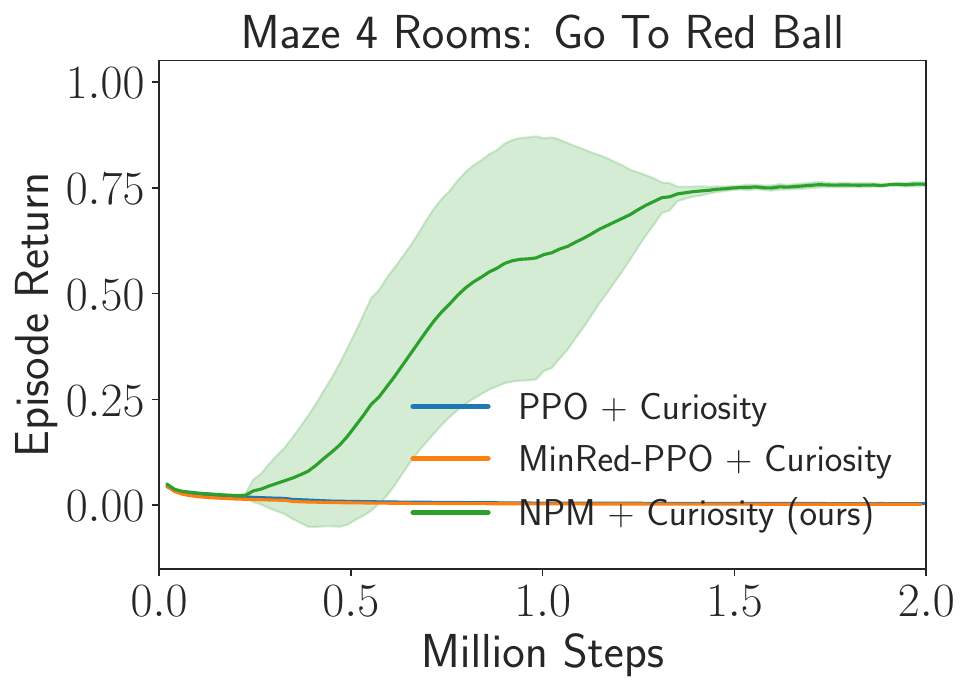}\label{fig: sub_figure1}}
    \subfloat{\includegraphics[scale=0.20]{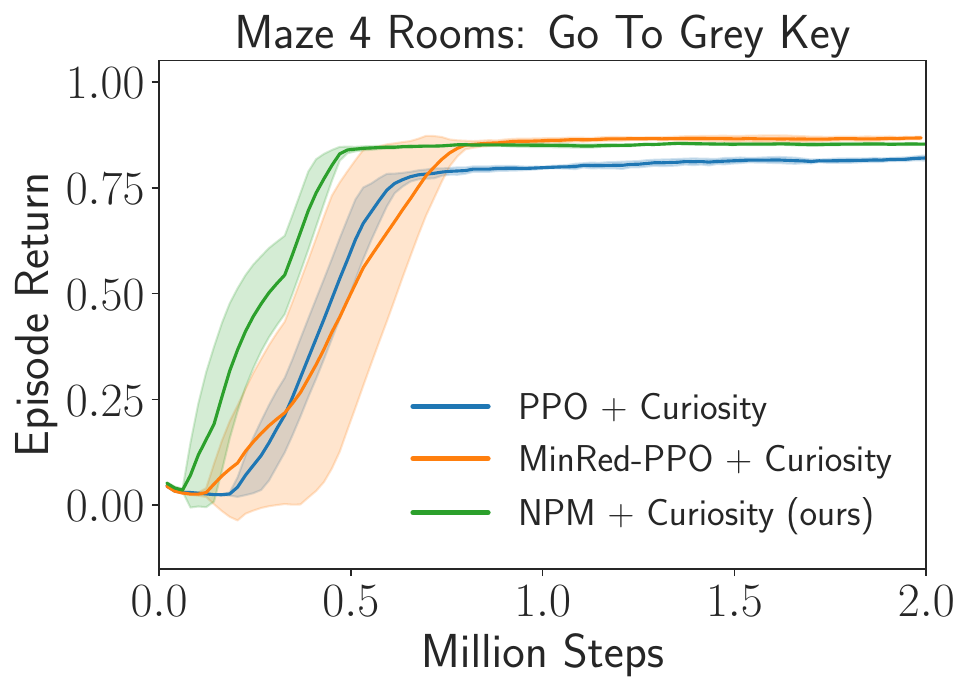}\label{fig: sub_figure1}}
    \subfloat{\includegraphics[scale=0.2]{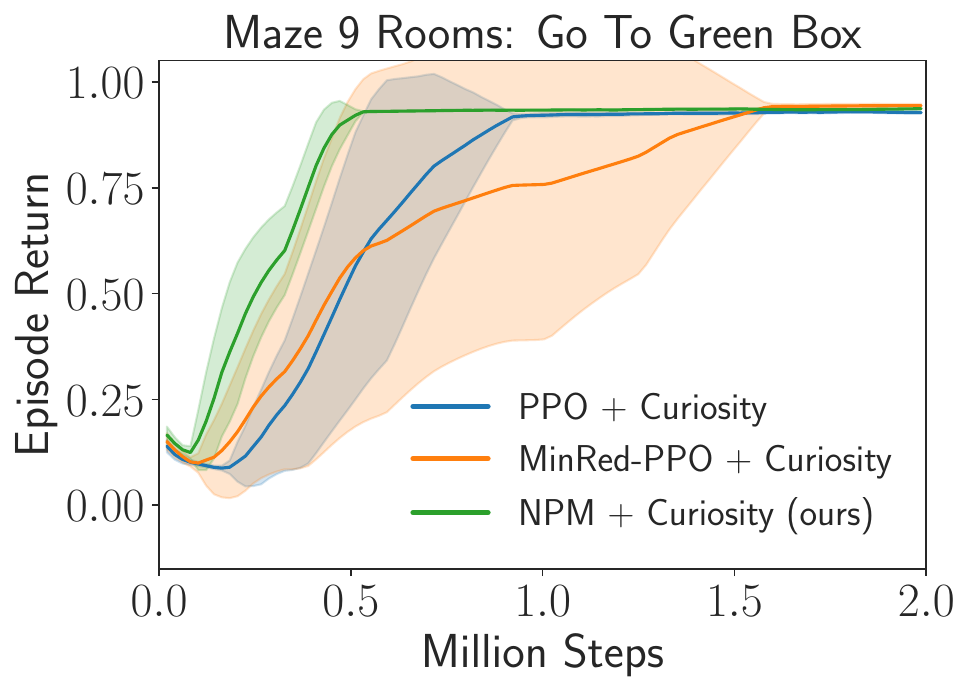}\label{fig: sub_figure3}}
    \subfloat{\includegraphics[scale=0.2]{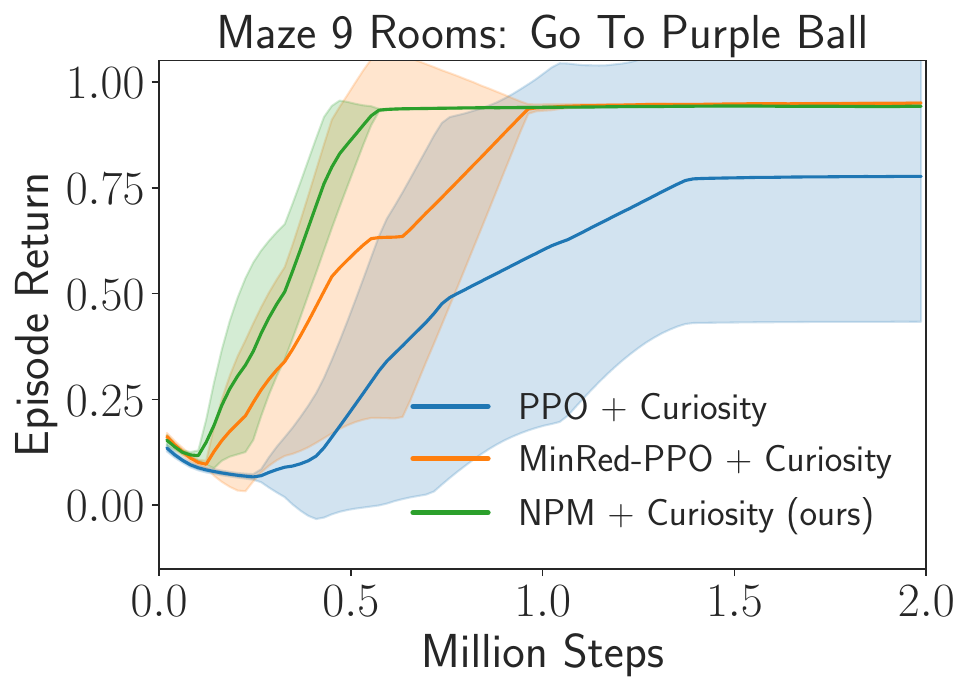}\label{fig: sub_figure3}}
    \caption{Above: The experimental results in the state-dependent action redundancy tasks under five random seeds.
    Below: Transfer results of the learned mask in the Maze 4 Rooms and Maze 9 Rooms.
    }
    \label{fig: task3 result}
\end{figure*}

\paragraph{Combined Action Redundancy.}
We consider the continuous-state maze task consistent with~\cite{chandak2019learning}, shown in Figure~\ref{fig: task description}.
Unlike the Synthetic Action Redundancy, the agent has $m$ equally spaced actuators~(each actuator moves the agent in the direction the actuator is pointing towards), and it can choose whether each actuator should be on or off.
Therefore, the size of the action set is exponential in the number of actuators, that is, $|\mathcal{A}|=2^m$.
The net outcome of an action is the vectorial summation of the displacements associated with the selected actuators.
To earn the reward, the agent must achieve the goal~(star) within 150 steps.



\paragraph{State-dependent Action Redundancy.}
We consider the Minigrid environments~\cite{minigrid,chevalier2018babyai} as the state-dependent action redundancy tasks.
As shown in Figure~\ref{fig: task3 setting} in Appendix, there are common identifiers among these tasks: agent, box, key, ball, and door.
We need to control the agent to complete different tasks.
For example, PutNextLocal is putting an object next to another object inside a single room with no doors.
UnlockPickup is picking up a box that is placed in another room behind a locked door.
The Maze 4 and 9 Rooms control the agent to go to an object, and the object may be in another room.
The action space is $\mathcal{A}=$\{Turn Left, Turn Right, Move Forward, Pick Up, Drop, Toggle, Noop\}.
Note that there are state-dependent redundancies in these tasks.
For instance, the agent can pick up an object only when this object is in front of the agent.
Otherwise, the pick-up action will not work, and the agent will stay still.
Therefore, we need to eliminate the action redundancy to complete these tasks efficiently.

\paragraph{Baselines:}
We compare NPM against the state-of-the-art baseline, MinRed~\cite{baram2021action}, which maximizes the next states' entropy to minimize action redundancy.
In this paper, we compare our method with two modified algorithms based on MinRed, named MinRed-DQN and MinRed-PPO. 
In addition, we compare our method with DQN~\cite{mnih2015human}, PPO~\cite{schulman2017proximal}, and PPO with curiosity reward~\cite{pathak2017curiosity}.

\subsection{Experimental Results}
\paragraph{Results of Synthetic Action Redundancy.}
We compare NPM with MinRed-DQN and DQN in the Four-Rooms task with synthetic action redundancy $n=\{1, 8, 16, 32\}$, which is shown in Figure~\ref{fig: FourRoomsexp}.
The experimental results show that with the increase of $n$, DQN converges more slowly or collapses.
The performance of MinRed-DQN will drop if $n$ is large.
Differently, NPM converges quickly and achieves better performance.

Further, we visualize the similarity factor matrix $M(s_t, \cdot, \cdot)$, where every element represents KL-divergence between two actions defined in Definition~\ref{def: similarity factor}.
The experimental results in Figure~\ref{fig: FourRoomsexp} show that the KL value between redundant actions is 0.
This demonstrates that the learned similarity factor model successfully captures the action structure, and we can use it to mask the redundant actions.

\begin{figure*}[t]
    \centering
    \subfloat{\includegraphics[scale=0.4]{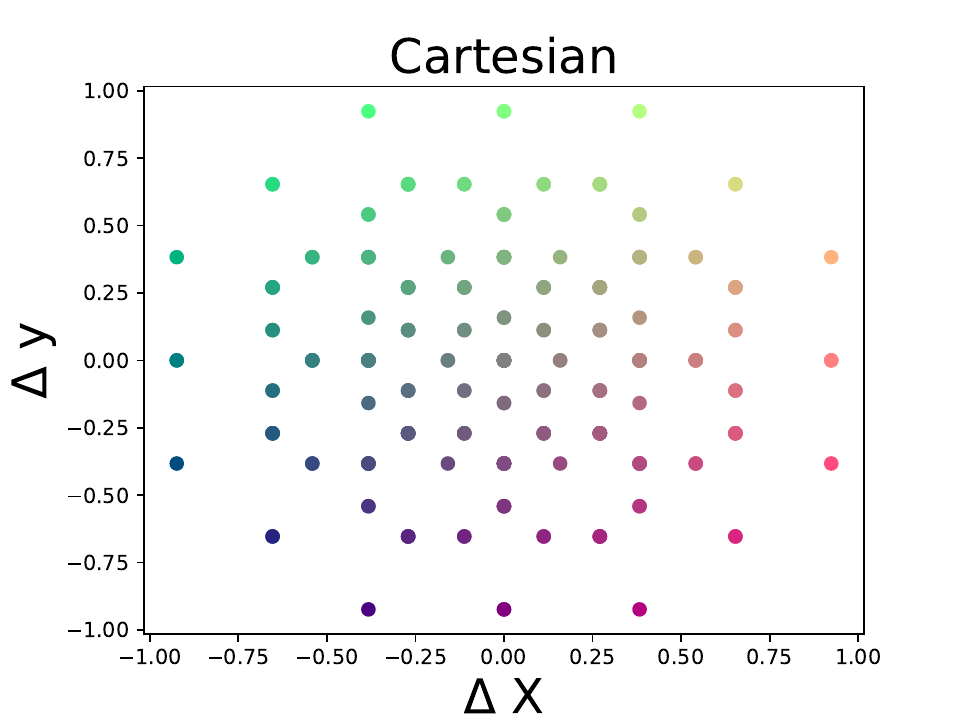}\label{fig: sub_figure1}}
    \subfloat{\includegraphics[scale=0.4]{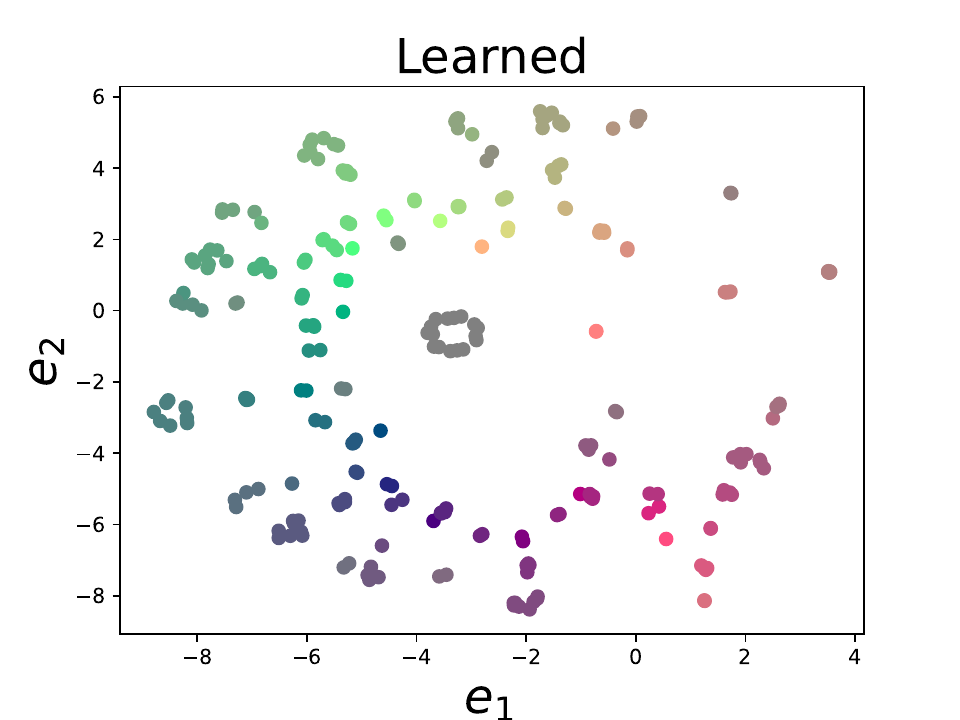}\label{fig: sub_figure2}}
    \caption{Left: 2-D representations for the displacements in the Cartesian coordinate caused by each action.
    Right: t-SNE projection based on similarity factor.
    We find that the actions with similar effects gather together based on the similarity factor.  The relative position of points of different colors in the Cartesian coordinate is consistent with the t-SNE projection result.
    }
    \label{fig: task2 visual}
\end{figure*}

\paragraph{Results of Combined Action Redundancy.}
In the maze, we compare NPM with MinRed-DQN and DQN with $|\mathcal{A}|=\{2^4, 2^6, 2^8\}$.
The experimental results in Figure~\ref{fig: mazeexp} show that NPM performs better than MinRed-DQN and DQN.
We further visualize the similarity factor.
Specifically, we first visualize the displacement in the Cartesian coordinate caused by each action.
We mark actions with a similar effect as similar colors~(There are some overlapping points).
Then, we calculate the t-SNE projection based on the similarity factor shown in Figure~\ref{fig: task2 visual}.
The experimental results show two interesting phenomenons:
The actions with similar effects gather together based on the similarity factor. 
Second, the relative position of points of different colors in the Cartesian coordinate is consistent with the t-SNE projection result.
This demonstrates that the relationship between different actions is also revealed in the learned representations.

\paragraph{Results of State-dependent Action Redundancy.}
We compare NPM with PPO and Minred+PPO with the curiosity reward.
The experimental result in Figure~\ref{fig: task3 result} shows NPM achieves superior performance than baselines.
In addition, we test the transfer ability of the learned mask.
Specifically, we train the action mask only in the ``Maze 4 Rooms: Go To Green Box'' and ``Maze 9 Rooms: Go to Blue Key'' in the first phase.
Next, we respectively train the RL policy with the corresponding action masks in the second phase in the ``Maze 4 Rooms: Go To Red Ball, Grey Key'' and ``Maze 9 Rooms: Go to Blue Key, Purple Ball''.
The experimental results in Figure~\ref{fig: task3 result} show that the learned mask of NPM can transfer across multi-tasks. 

Further, we visualize the learned similarity factor.
The results in Figure~\ref{fig: task3 visual} and Figure~\ref{fig: different action clusters} in Appendix show that the agent has different action clusters in various state positions, which can be successfully captured based on the learned similarity factor.

\subsection{Ablation Study}
\paragraph{How to select $\epsilon$} is an important issue in our algorithm.
As shown by Lemma~\ref{lemma1}, to limit the performance error to an acceptable level, we need to adjust the value of $\epsilon$ as the reward range $[0, r_{\max}]$ in the environment increases.
In typical scenarios where the range of reward is not particularly large, it is appropriate to choose the range of hyperparameters $\epsilon$ between 0.05 and 0.5 in practical implementation.
We conduct the ablation studies for $\epsilon$ $\in$ \{0.01, 0.05, 0.1, 0.5, 1, 5\}.
The results in Figure~\ref{fig: ablation_epsilon} in Appendix~\ref{appendix: addi exp} show that NPM performs robust when $\epsilon$ $\in$ (0.05, 0.5).

\paragraph{Modified inverse dynamic model:}
We compare the modified inverse dynamic model with the original model in the combined action redundancy domain~(maze task).
We select various action numbers $n$ in the maze task.
The experimental results in Figure~\ref{fig: ablation_modified} in Appendix~\ref{appendix: addi exp} show that the modified inverse dynamic model~(named modified) learns faster and achieves higher accuracy than the original dynamic model~(named original).
The result is consistent with the statement in section~\ref{subsec: inverse model}.

\paragraph{Curiosity Reward:}
We conduct ablation studies for the curiosity reward used in the state-dependent action redundancy tasks.
The results in Figure~\ref{fig: ablation_curiosity} in Appendix~\ref{appendix: addi exp} show that the curiosity reward is essential in complex tasks.
Further, we visualize the visitation locations in the Unlock-Pickup task.
The result in Figure~\ref{fig: heatmap} in Appendix shows that agent driven by the curiosity reward succeeds in adequate exploration.

\paragraph{Extend to Atari Tasks:}
We conduct experiments on another state-dependent action redundancy tasks~(e.g., Atari 2600).
The experimental results in Figure~\ref{fig:DemonAttack cluster} and Figure~\ref{fig: Atari result} in Appendix show our method can effectively remove the redundant actions and improve the current baselines.


\section{Conclusion}
In this paper, we construct a framework to reveal the underlying structure of action space without prior knowledge.
Our theoretical analysis has yielded a crucial finding: the policy derived from our framework has minimal impact on the performance of the original policy.
Building upon these insights, we propose a simple yet efficient algorithm, known as NPM.
In practice, our approach consistently outperforms existing methods across a wide range of tasks.
In future, we will explore the extension of NPM into continuous action spaces and its potential applications in robotics.


\section{Acknowledgments}
This work is supported by National Natural Science Foundation of China under Grant No. 62192751, in part by Key R\&D Project of China under Grant No. 2017YFC0704100, the 111 International Collaboration Program of China under Grant No. BP2018006, and in part by the BNRist Program under Grant No. BNR2019TD01009, the National Innovation Center of High Speed Train R\&D project (CX/KJ-2020-0006), in part by the InnoHK Initiative, The Government of HKSAR; and in part by the Laboratory for AI-Powered Financial Technologies.

\bibliography{aaai24}


\clearpage
\appendix

\onecolumn
\section{Algorithm}
\label{appendix: alg}

\begin{algorithm}[h]
\caption{No Prior Mask Reinforcement Learning}
\label{alg: NPM}
\begin{algorithmic}[1]
\STATE {\bf Inputs:} Initialize policy $\pi_{\phi}$, inverse dynamics model $P^{\rm inv}_{\psi}$, N-value network $N_{\theta}$ with $\phi_0, \psi_0, \theta_0$.
\STATE {\bf Hyper-Parameter:} Set the hyper-parameter $\epsilon$ and N-value network update interval $T$.
\STATE \textcolor{purple}{Phase 1: Train No Prior Mask}
\FOR{iteration $i=0,1,2,...$}
\STATE Rollout with Policy( random initialized) and store data into buffer $\mathcal{D}$
\STATE Sample transitions from buffer $\mathcal{D}$
\STATE Update inverse dynamics model parameters $\psi$ based on the Equation~\ref{eq: mask}
\IF{ $i$ \% $T == 0$ }
{
\STATE Update N-value network parameters $\theta$ based on the Equation~\ref{eq: N-value}
}
\ENDIF
\IF{ task is complex }
{
\STATE Update policy parameters based on the curiosity reward
}
\ENDIF
\ENDFOR

\STATE \textcolor{purple}{Phase 2: Optimize Policy based on the No Prior Mask}
\FOR{$t=0,1,2,...$}
\STATE Evaluate the similarity factor matrix $M(s_t, \cdot, \cdot)$ using N-Value network $N_{\theta}$
\STATE Construct minimal action space $\mathcal{A}_{s}^{\epsilon} = \{ a_{A_0}, a_{A_1},...,a_{A_{l-1}}\}$
\STATE Sample action $a_{A_m}$ based on the Equation~\ref{eq: value-based}
or Equation~\ref{eq: pi mask}
\STATE Execute action $a_{A_m}$ and store data into buffer $\mathcal{D}$
\STATE Optimiaze $\pi_{\phi}$ using reinforcement learning algorithm
\ENDFOR
\STATE {\bf Return:} Policy parameters $\phi$ and N-value network parameters $\theta$ \\
\end{algorithmic}
\end{algorithm}

\clearpage
\section{Proofs}
\label{appendix: proof}
\subsection{Proof of Lemma~\ref{lemma1}}
\begin{lemma*}
\label{lemma1-app}
For any policy $\pi$ and $A_{s,\epsilon}^m$, we can choose arbitrarily one action $a_{A_m} \in A_{s,\epsilon}^m$ to represent the sub-action space. We denote
\begin{align}
\pi_{A_{s,\epsilon}^m}(a \mid s)= \begin{cases}\sum_{a^{\prime} \in A_{s,\epsilon}^m} \pi\left(a^{\prime} \mid s\right) & , a = a_{A_m} \\ 0 &, a \in A_{s,\epsilon}^m \backslash \{ a_{A_m}\} \\ \pi(a \mid s) & , \text {\rm o.w }\end{cases}
\end{align}

then
\begin{align}
\left\|V^\pi-V^{\pi_{A_{s,\epsilon}^m}}\right\|_{\infty} \leq \frac{\gamma r_{max}}{(1-\gamma)^2} \sqrt{2\epsilon}
\end{align}
\end{lemma*}




\begin{proof}
Let $\mathcal{M}$ be the MDP defined by $(\mathcal{S}, \mathcal{A}, P, r, \gamma)$, and let $\hat{\mathcal{M}}$ be the MDP defined by $(\mathcal{S}, \mathcal{A}, \hat{P}, r, \gamma)$, where
$$
\hat{P}\left(s^{\prime} \mid s, a\right)= \begin{cases} P \left(s^{\prime} \mid s, a\right) & , a = a_{A_m} \\ P \left(s^{\prime} \mid s, a_{A_m}\right) &, a \in A_{s,\epsilon}^m \backslash \{ a_{A_m}\} \\ P \left(s^{\prime} \mid s, a\right) & , \text {\rm o.w }\end{cases}
$$
By definition of $\hat{\mathcal{M}}$ and $\pi_{A_{s,\epsilon}^m}$ we have that $\hat{P}^\pi=P^{\pi_{A_{s,\epsilon}^m}}$, which transfer the policy evaluation of $\pi_{A_{s,\epsilon}^m}$ on $\mathcal{M}$ to that of $\pi$ on $\hat{\mathcal{M}}$.

Next, we will prove the following equation, which shows that the total variation distance between $P$ and $\hat{P}$ will be bounded for all $s,a$.

\begin{equation}
\label{eq:tdleq}
  \sum_{s^{\prime}}\left|P\left(s^{\prime} \mid s, a\right)-\hat{P}\left(s^{\prime} \mid s, a\right)\right| = \begin{cases} 0 & , a = a_{A_m} \\ \sum_{s^{\prime}}\left|P\left(s^{\prime} \mid s, a\right)-P\left(s^{\prime} \mid s, a_{A_m}\right)\right| &, a \in A_{s,\epsilon}^m \backslash \{ a_{A_m}\} \\ 0 & , \text {\rm o.w }\end{cases}
\end{equation}

We note the following relationship between the total variation divergence and KL divergence (~\cite{10.1093/acprof:oso/9780199535255.001.0001}):

\begin{align*}
    D_{T V}(p \| q)^2 \leq \frac{1}{2} D_{\mathrm{KL}}(p \| q)
\end{align*}

where
\begin{align*}
    D_{T V}(p \| q)=\frac{1}{2} \sum_i\left|p_i-q_i\right|.
\end{align*}

Therefore, we have that for $a \in A_{s,\epsilon}^m \backslash \{ a_{A_m}\}$,

\begin{equation}
\label{eq:tdleq2}
  \sum_{s^{\prime}}\left|P\left(s^{\prime} \mid s, a\right)-P\left(s^{\prime} \mid s, a_{A_m}\right)\right| \leq 
  \sqrt{2 D_{\text{\rm KL}}(P\left(s^{\prime} \mid s, a\right) \| P\left(s^{\prime} \mid s, a_{A_m}\right)) }
   < \sqrt{2 \epsilon}
\end{equation}

where in the step, we use the condition
$$M(s_t,a_i,a_j) = D_{\text{\rm KL}}(P\left(s^{\prime} \mid s, a_i\right) \| P\left(s^{\prime} \mid s, a_j\right)) < \epsilon , \forall a_i, a_j \in A_{s,\epsilon}^m$$

Therefore, we have that
\begin{equation}
    \sum_{s^{\prime}}\left|P\left(s^{\prime} \mid s, a\right)-\hat{P}\left(s^{\prime} \mid s, a\right)\right| < \sqrt{2 \epsilon}, \forall s, a
\end{equation}






Similar to the proof of Lemma 1 in ~\cite{Tennenholtz2019TheNL}, we use the following result proven in Lemma B.2 in ~\cite{janner2019trust} :

\begin{lemma}
\label{TVDeq}
Let $\mathcal{M}, \hat{\mathcal{M}}$ be MDPs as defined above. If
$$
\sum_{s^{\prime}}\left|P\left(s^{\prime} \mid s, a\right)-\hat{P}\left(s^{\prime} \mid s, a\right)\right|<\delta, \forall s, a
$$
then
$$
\sum_{s_t}\left|P\left(s_t \mid s_0\right)-\hat{P}\left(s_t \mid s_0\right)\right|<t \delta
$$
\end{lemma}

By definition of $\hat{\mathcal{M}}$ and $V^{\pi_{A_{s,\epsilon}^m}}$ we have that $\hat{P}^\pi=P^{\pi_{A_{s,\epsilon}^m}}$. Then, for all $s \in \mathcal{S}$
$$
\begin{aligned}
& \left|V^\pi(s)-V^{\pi_{A_{s,\epsilon}^m}}\right| \\
& =\left|\mathbb{E}_{P^\pi}\left(\sum_{t=0}^{\infty} \gamma^t r\left(s_t\right) \mid s_0=s\right)-\mathbb{E}_{\hat{P}^\pi}\left(\sum_{t=0}^{\infty} \gamma^t r\left(s_t\right) \mid s_0=s\right)\right| \\
& =\left|\sum_{t=0}^{\infty} \gamma^t\left[\mathbb{E}_{P^\pi}\left(r\left(s_t\right) \mid s_0=s\right)-\mathbb{E}_{\hat{P}^\pi}\left(r\left(s_t\right) \mid s_0=s\right)\right]\right| \\
& \leq \sum_{t=0}^{\infty} \gamma^t\left|\mathbb{E}_{P^\pi}\left(r\left(s_t\right) \mid s_0=s\right)-\mathbb{E}_{\hat{P}^\pi}\left(r\left(s_t\right) \mid s_0=s\right)\right| .
\end{aligned}
$$
Writing the above explicitly we get
\begin{equation}
\label{eq:verror}
\begin{aligned}
&\left|V^\pi(s)-V^{\pi_{A_{s,\epsilon}^m}}(s)\right| \\
&\leq \sum_{t=0}^{\infty} \gamma^t\left|\sum_{s_t} r(s_t)\left(P\left(s_t \mid s_0=s\right)-\hat{P}\left(s_t \mid s_0 =s\right)\right)\right| 
\end{aligned}
\end{equation}

Next
$$
\begin{aligned}
& \sum_{s_t} r\left(s_t\right)\left(P\left(s_t \mid s_0=s\right)-\hat{P}\left(s_t \mid s_0 =s\right)\right) \\
& \leq \sum_{s_t} r\left(s_t\right)\left|P\left(s_t \mid s_0=s\right)-\hat{P}\left(s_t \mid s_0=s\right)\right| \\
& \leq r_{\max} \sum_{s_t}\left|P\left(s_t \mid s_0=s\right)-\hat{P}\left(s_t \mid s_0=s\right)\right|
\end{aligned}
$$

Using Lemma~\ref{TVDeq}  we get
$$
\begin{aligned}
&\left|V^\pi(s)-V^{\pi_{A_{s,\epsilon}^m}}(s)\right| \\
&\leq \sum_{t=0}^{\infty} \gamma^t\left|\sum_{s_t} r(s_t)\left(P\left(s_t \mid s_0=s\right)-\hat{P}\left(s_t \mid s_0 =s\right)\right)\right| \\
&\leq \sum_{t=0}^{\infty} \gamma^t r_{\max} \sum_{s_t}\left|P\left(s_t \mid s_0=s\right)-\hat{P}\left(s_t \mid s_0=s\right)\right| \\
&\leq \sum_{t=0}^{\infty} \gamma^t r_{\max}t \sqrt{2\epsilon}
\end{aligned}
$$
The proof follows immediately due to
$$
\sum_{t=0}^{\infty} \gamma^t t= \frac{\gamma}{(1-\gamma)^2}
$$
for $|\gamma|<1$. Then, we have
$$\left\|V^\pi-V^{\pi_{A_{s,\epsilon}^m}}\right\|_{\infty} \leq \frac{\gamma r_{\max}}{(1-\gamma)^2} \sqrt{2\epsilon}$$

\end{proof}

\subsection{Proof of Lemma~\ref{lemma2}}

\begin{lemma*}
\label{lemma2-app}
The similarity factor $M$ can be divided by two terms with the same form $N$, which we refer to as N-value network: 
\begin{align*}
M(s_t,a_i,a_j) = N(s_t,a_i,a_i) - N(s_t,a_i,a_j)
\end{align*}
where
\begin{align*}
N(s_t,a_i,a_j) = \mathop{E} \limits_{s_{t+1} \sim P(\cdot|s_t,a_i)}\log \left[ \frac{P^{\pi}(a_j|s_t,s_{t+1})}{\pi(a_j|s_t)} \right]
\end{align*}
For $i = j$, then we have the latter term
\begin{align*}
N(s_t,a_i,a_i) = \mathop{E} \limits_{s_{t+1} \sim P(\cdot|s_t,a_i)}\log \left[ \frac{P^{\pi}(a_i|s_t,s_{t+1})}{\pi(a_i|s_t)} \right]
\end{align*}
\end{lemma*}


\begin{proof}
Following the definition of KL-divergence, we have
\begin{align*}
M(s_t,a_i,a_j) 
&=D_{\text{KL}}(P(s_{t+1}|s_t,a_i)||P(s_{t+1}|s_t,a_j)) 
\\
&= E_{s_{t+1} \sim P(\cdot|s_t,a_i}) \log\left[\frac{P(s_{t+1}|s_t,a_i)}{P(s_{t+1}|s_t,a_j)}\right]
\end{align*}
We insert $P^{\pi}(s_{t+1}|s_t)$ with respect to a stochastic policy $\pi$:

\begin{align*}
\frac{P(s_{t+1}|s_t,a_i)}{P(s_{t+1}|s_t,a_j)} =
\frac{P(s_{t+1}|s_t,a_i) / P^{\pi}(s_{t+1}|s_t)}{P(s_{t+1}|s_t,a_j)/P^{\pi}(s_{t+1}|s_t)} 
\end{align*}

where
\begin{align*}
\frac{P(s_{t+1}|s_t,a_i)}{P^{\pi}(s_{t+1}|s_t)} &= \frac{P^{\pi}(s_{t+1},s_t,a_i) / P^{\pi}(s_t,a_i)}{P^{\pi}(s_{t+1},s_t)/P^{\pi}(s_t)} \\&= \frac{P^{\pi}(a_i|s_t,s_{t+1})}{\pi(a_i|s_t)}
\end{align*}

We can derive the equation for the denominator
in the same way. 

\begin{align*}
\frac{P(s_{t+1}|s_t,a_j)}{P^{\pi}(s_{t+1}|s_t)} &= \frac{P^{\pi}(s_{t+1},s_t,a_j) / P^{\pi}(s_t,a_j)}{P^{\pi}(s_{t+1},s_t)/P^{\pi}(s_t)} \\&= \frac{P^{\pi}(a_j|s_t,s_{t+1})}{\pi(a_j|s_t)}
\end{align*}

Then, we have 
\begin{align*}
 &M(s_t,a_i,a_j) = N(s_t,a_i,a_i) - N(s_t,a_i,a_j)
\end{align*}
where
$$
N(s_t,a_i,a_j) = \mathop{E} \limits_{s_{t+1} \sim P(\cdot|s_t,a_i)} \log \left[ \frac{P^{\pi}(a_j|s_t,s_{t+1})}{\pi(a_j|s_t)} \right].$$
\end{proof}

\clearpage
\section{Additional Experiments}
\label{appendix: addi exp}

\paragraph{How to select $\epsilon$} is an important issue in our algorithm.
If $\epsilon$ is too large, the constraint is loose, and we will filter out enormous valid actions.
If $\epsilon$ is too small, the constraint is tight, and we cannot identify the redundant actions.
Therefore, we conduct the ablation studies for $\epsilon$ $\in$ \{0.01, 0.05, 0.1, 0.5, 1, 5\}.
The results in Figure~\ref{fig: ablation_epsilon}  show that NPM performs robust when $\epsilon$ $\in$ (0.05, 0.5).

\begin{figure*}[h]
    \centering
    \subfloat{\includegraphics[scale=0.2]{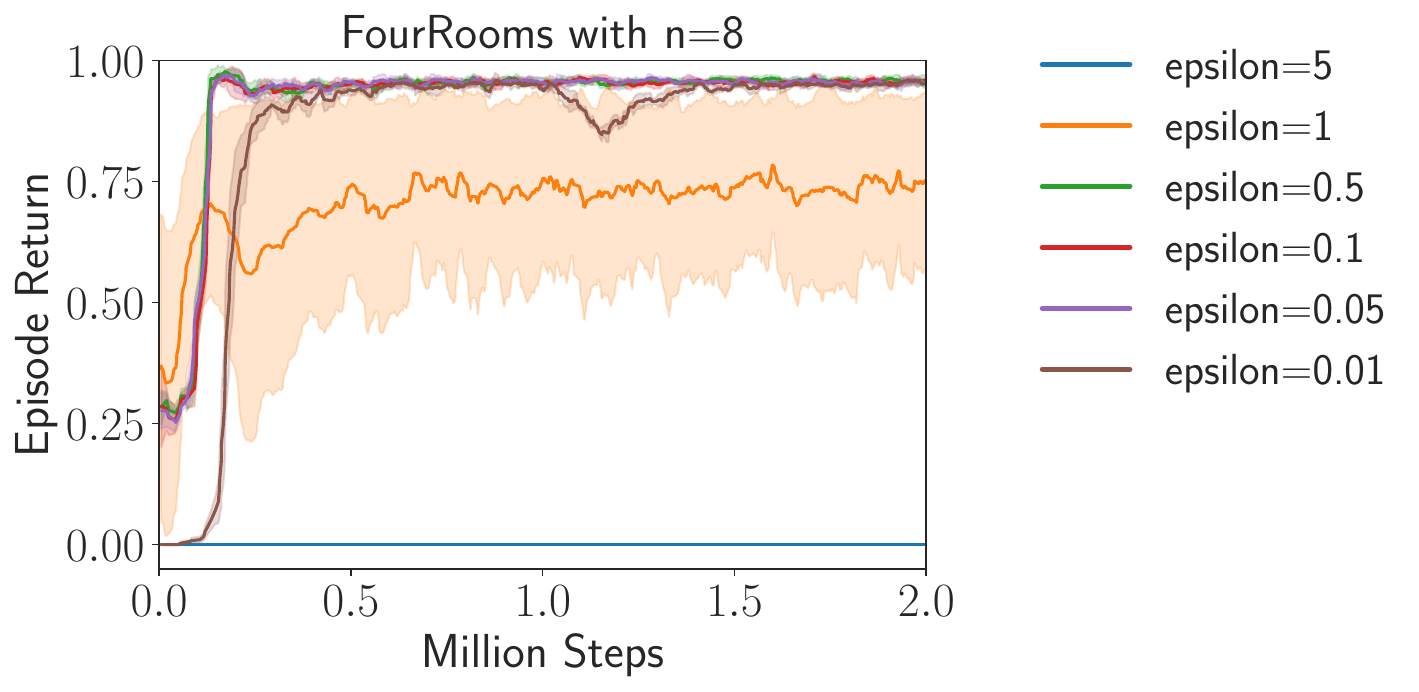}\label{fig: sub_figure1}}
    \subfloat{\includegraphics[scale=0.2]{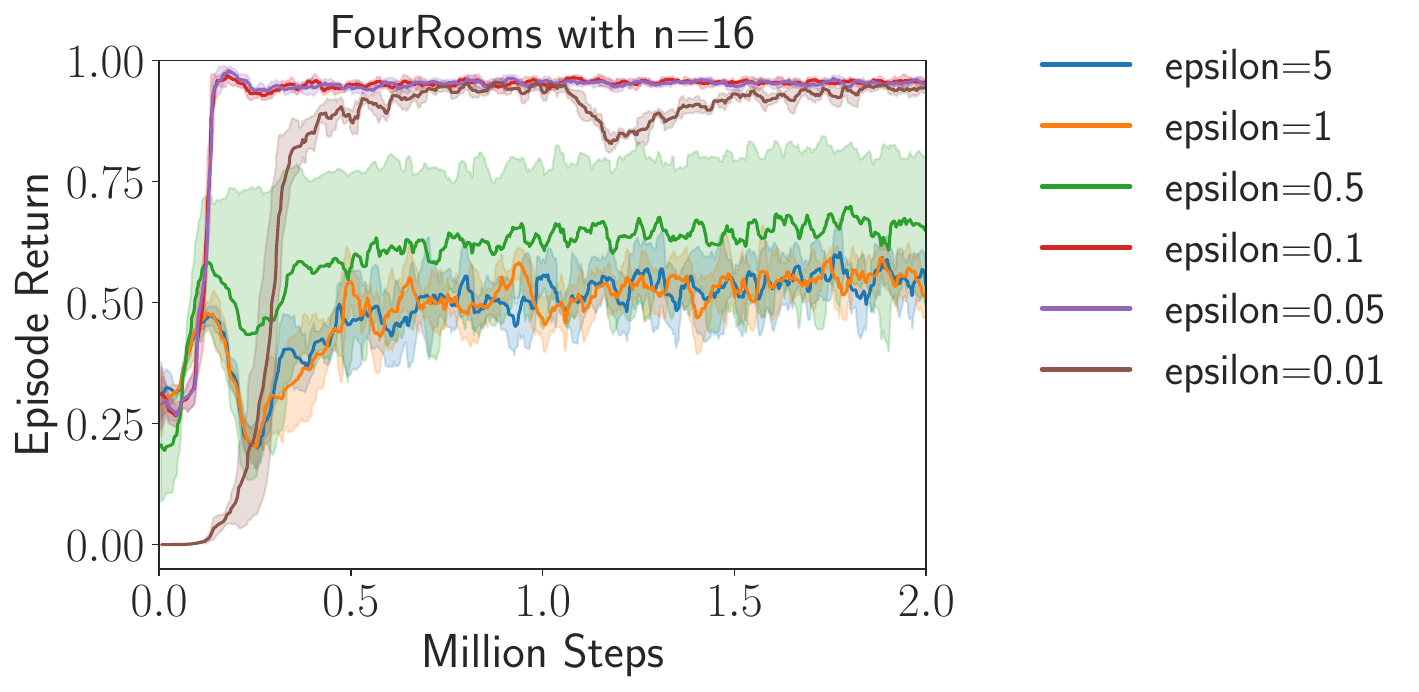}\label{fig: sub_figure2}}
    \subfloat{\includegraphics[scale=0.2]{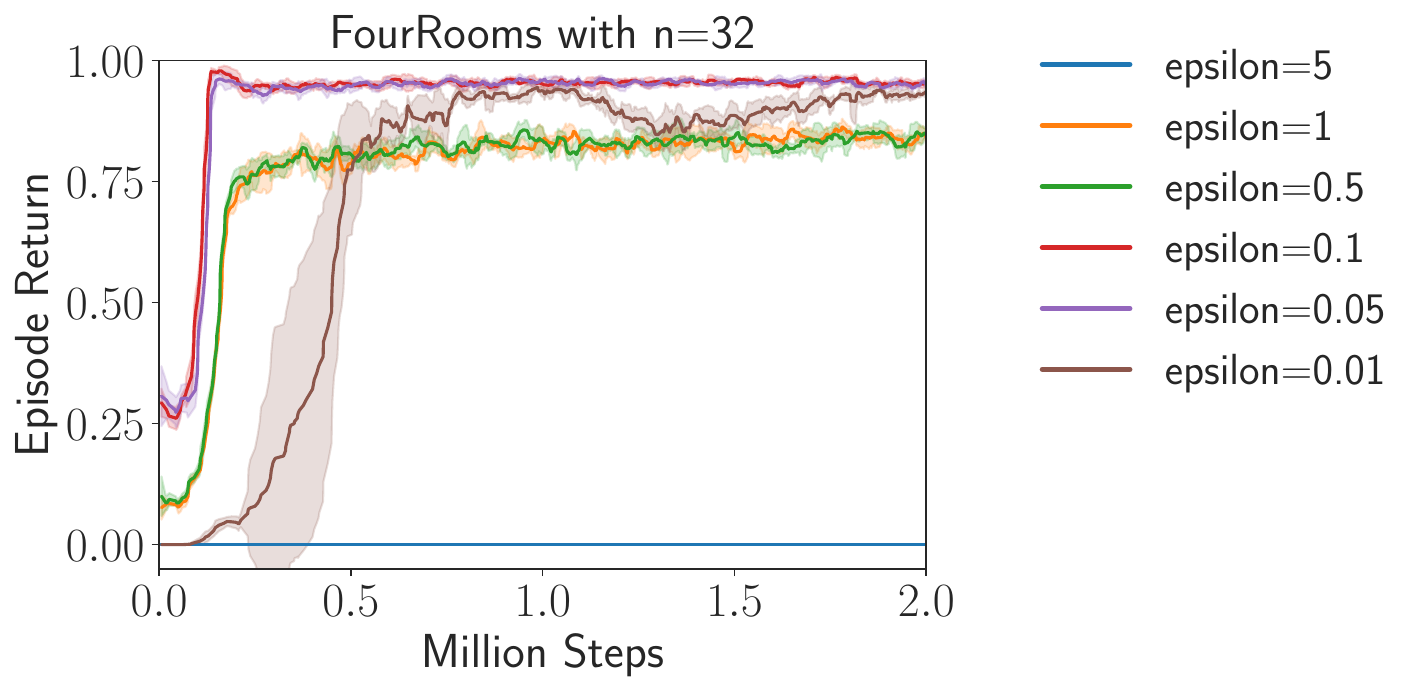}\label{fig: sub_figure3}}
    \caption{Ablation study with $\epsilon$ $\in$ \{0.01, 0.05, 0.1, 0.5, 1, 5\}.}
    \label{fig: ablation_epsilon}
\end{figure*}

\paragraph{Modified inverse dynamic model:}
We compare the modified inverse dynamic model with the original model in the combined action redundancy domain~(maze task).
We select various action numbers $n$ in the maze task.
The experimental results in Figure~\ref{fig: ablation_modified} show that the modified inverse dynamic model~(named modified) learns faster and achieves higher accuracy than the original dynamic model~(named original).
The result is consistent with the statement in section~\ref{subsec: inverse model}.

\begin{figure*}[h]
    \centering
    \subfloat{\includegraphics[scale=0.25]{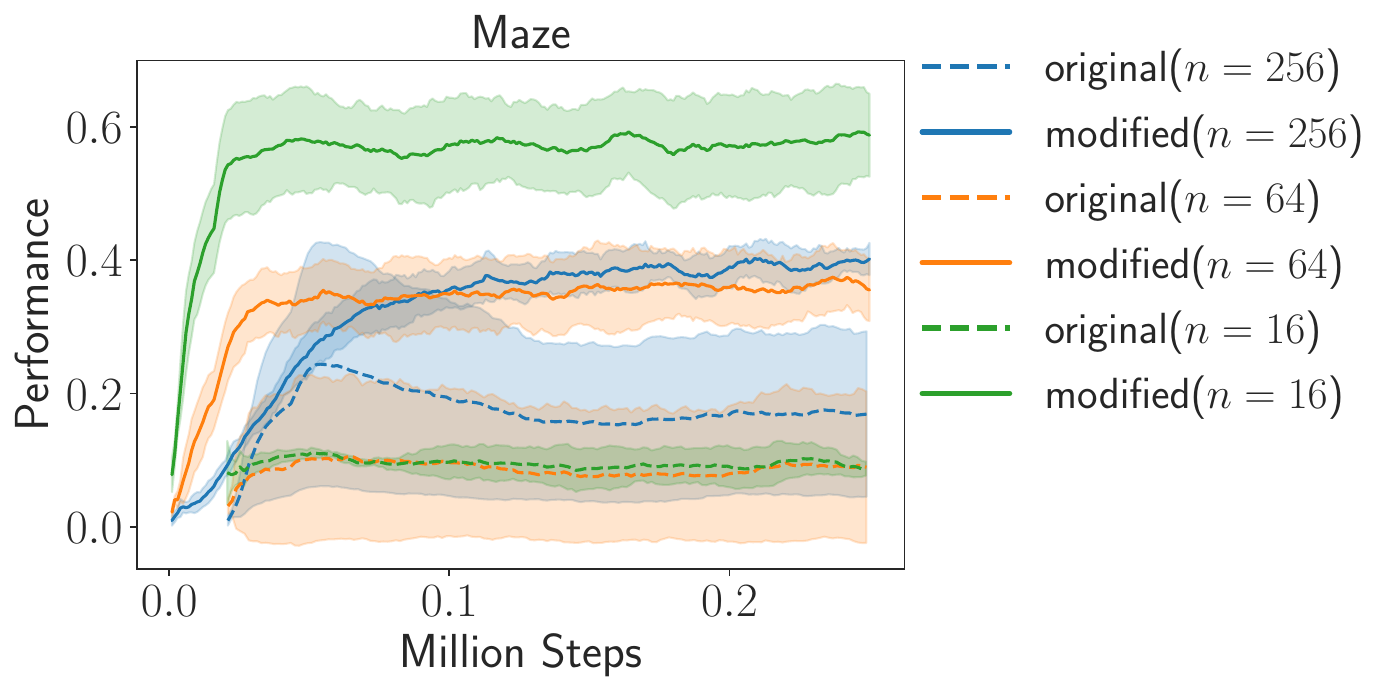}}
    \caption{Ablation study with the modified inverse dynamic model.}
    \label{fig: ablation_modified}
\end{figure*}

\paragraph{Curiosity Reward:}
We conducted ablation studies to analyze the impact of the curiosity reward within the context of state-dependent action redundancy tasks. The results presented in Figure~\ref{fig: ablation_curiosity} unequivocally demonstrate the indispensable role of the curiosity reward, particularly in challenging and intricate tasks.
Especially for the pre-training of action mask in phase 1, it is essential to collect data using varying policies (curiosity-driven) to train a model that applies to a broader state space.

\begin{figure*}[h]
    \centering
    \subfloat{\includegraphics[scale=0.25]{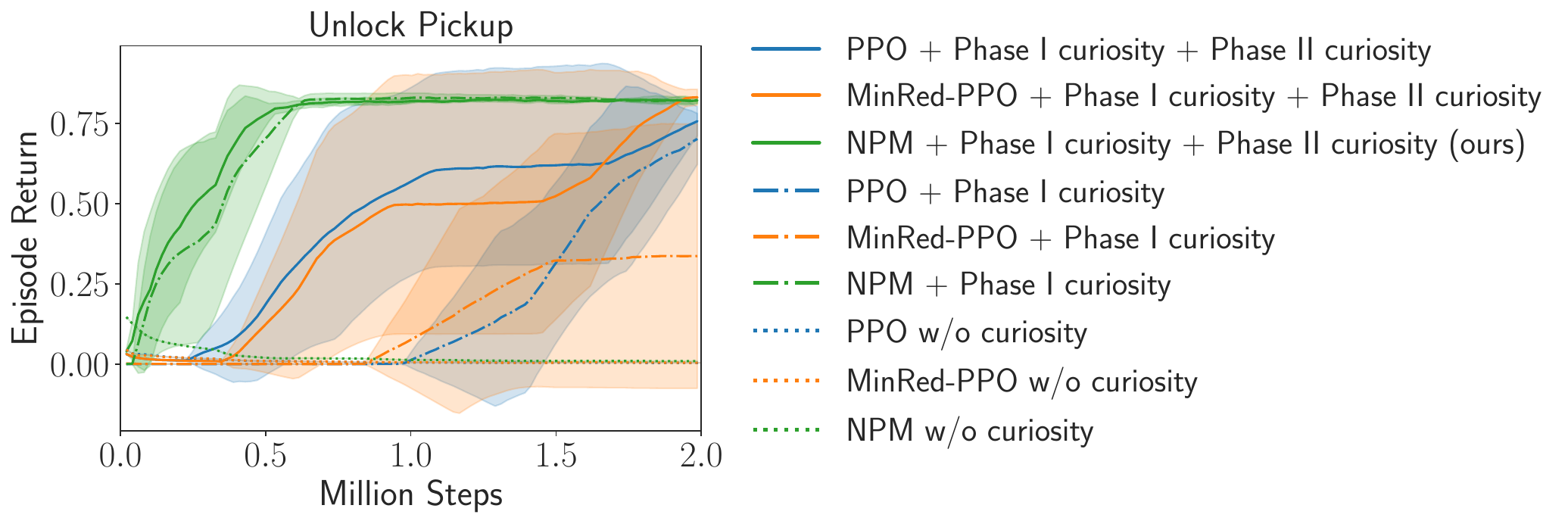}}
    \caption{Ablation study with curiosity.}
    \label{fig: ablation_curiosity}
\end{figure*}

\paragraph{Visitation Locations Heatmap:}
We offer a visual representation of visitation locations within the Unlock-Pickup task for both the fixed policy and the curiosity-driven policy. The outcome, illustrated in Figure~\ref{fig: heatmap}, illustrates the effectiveness of an agent motivated by the curiosity reward in achieving comprehensive exploration during the phase 1 pre-training stage.

\begin{figure*}[h]
     \centering
    \subfloat{\includegraphics[scale=0.25]{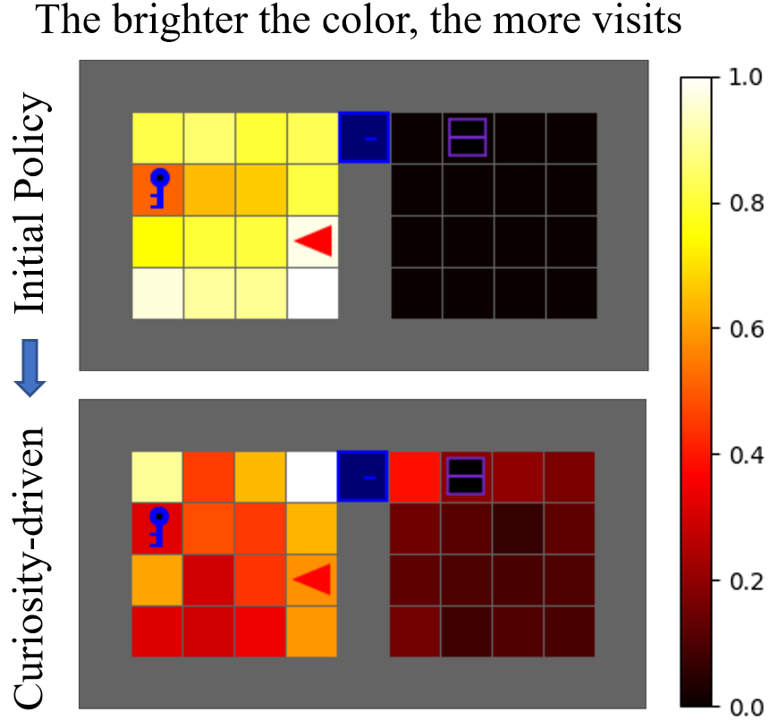}}
     \caption{Visitation locations heatmap.}
     \label{fig: heatmap}
 \end{figure*}

\paragraph{Soft Action Mask:}
We conducted additional experiments with soft masks, which do not entirely eliminate actions but assign probabilities based on their similarity measure.
Specifically, we calculate the average similarity for each action based on similarity factor matrix $M$: $$\bar{d}_{s_t,a_i} = \frac{\sum_{j \neq i} M(s_t,a_i,a_j)}{ |\mathcal{A}| - 1},$$ where $|\mathcal{A}|$ is the size of the action space. During phase 2 of training, we modify the policy distribution as follows: $$\pi_{\rm softmask}(a_i|s_t) = \frac{\pi(a_i|s_t)e^{\eta\bar{d}_{s_t,a_i}}}{Z},$$ where $\eta$ is a weight coefficient, and $Z$ is a normalization factor to maintain the probability distribution sum to 1.  By applying this modification, when an action is similar to many others, its probability of being sampled is reduced rather than being directly masked. The experimental results in Figure~\ref{fig: ablation_fig_Maze_9_Rooms.pdf} show that compared with hard masks, soft masks slow down the training process.

\begin{figure*}[h]
     \centering
    \subfloat{\includegraphics[scale=0.25]{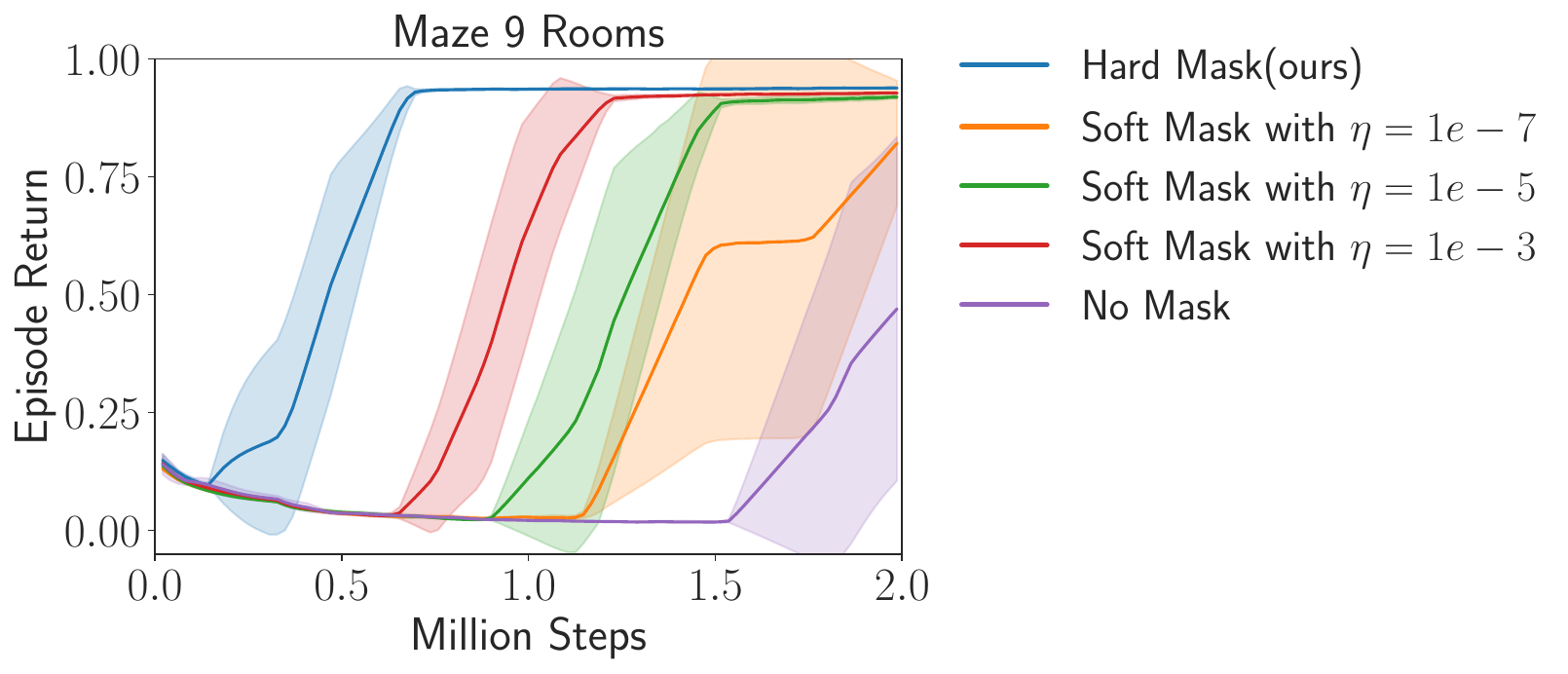}}
     \caption{The experimental results of Hard Mask, Soft Mask and No Mask.}
     \label{fig: ablation_fig_Maze_9_Rooms.pdf}
 \end{figure*}
 
\clearpage
\section{Minigrid environments}
We consider the Minigrid environments~\cite{minigrid,chevalier2018babyai} as the state-dependent action redundancy tasks.
As shown in Figure~\ref{fig: task3 setting}, there are common identifiers among these tasks: agent, box, key, ball, and door.
We need to control the agent to complete different tasks.
For example, PutNextLocal is putting an object next to another object inside a single room with no doors.
UnlockPickup is picking up a box that is placed in another room behind a locked door.
The Maze 4 and 9 Rooms control the agent to go to an object, and the object may be in another room.
The action space is $\mathcal{A}=$\{Turn Left, Turn Right, Move Forward, Pick Up, Drop, Toggle, Noop\}.
Note that there are state-dependent redundancies in these tasks.
For instance, the agent can pick up an object only when this object is in front of the agent.
Otherwise, the pick-up action will not work, and the agent will stay still.
Therefore, we need to eliminate the action redundancy to complete these tasks efficiently.
\begin{figure*}[h]
    \centering
    \includegraphics[scale=0.108]{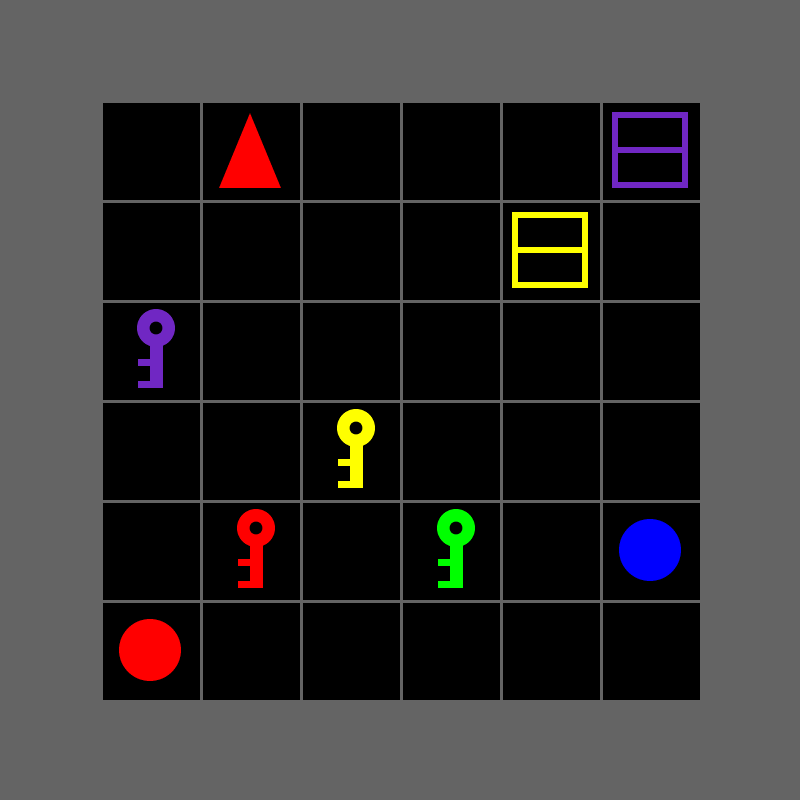}
    \hspace{1mm}
    \includegraphics[scale=0.1]{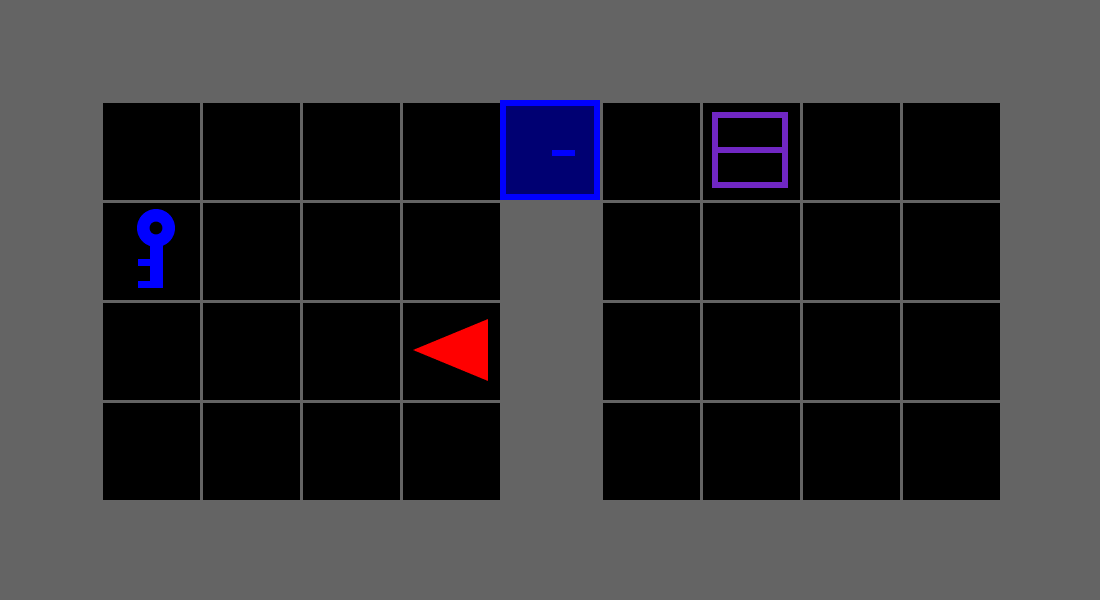}
    \hspace{1mm}
    \includegraphics[scale=0.058]{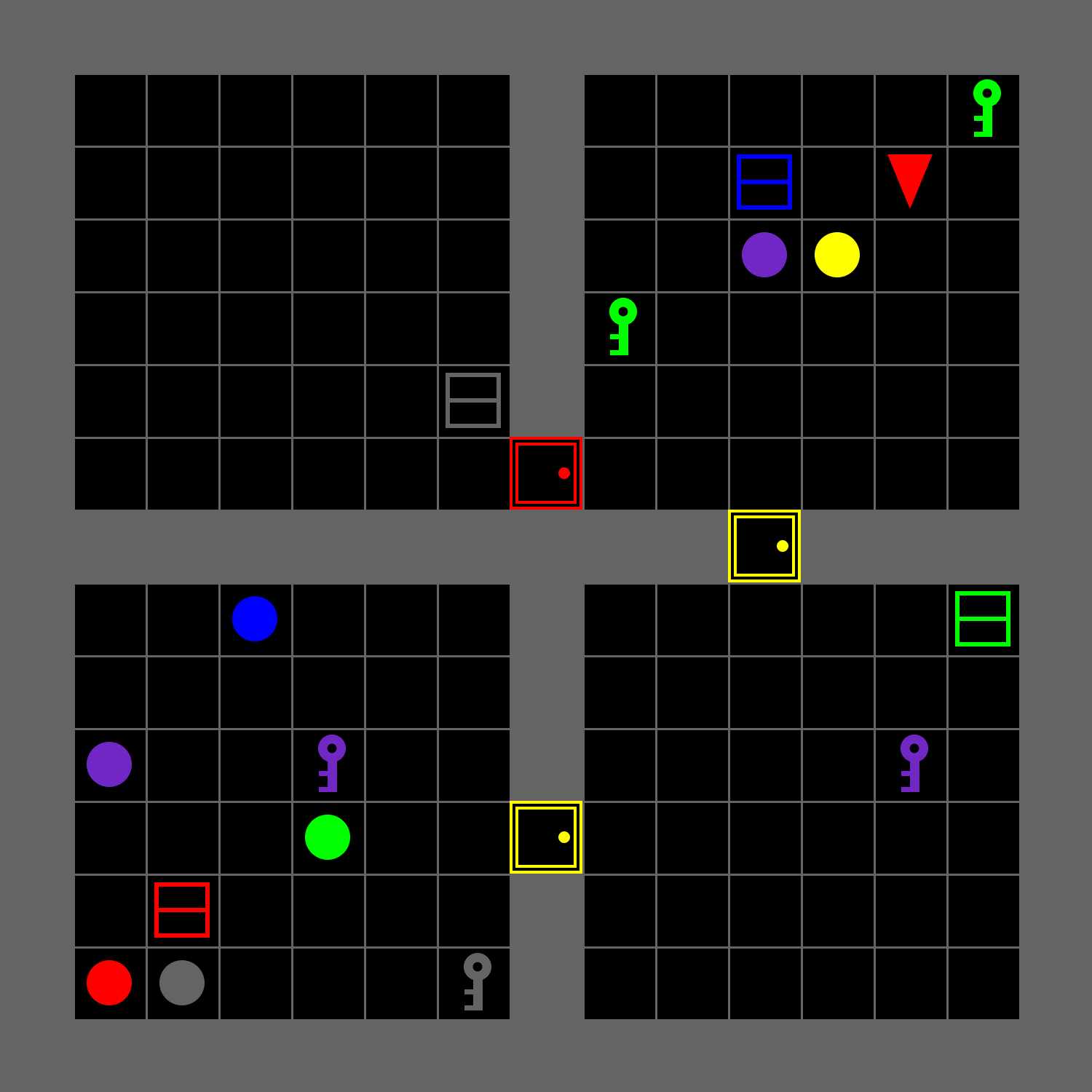}
    \includegraphics[scale=0.04]{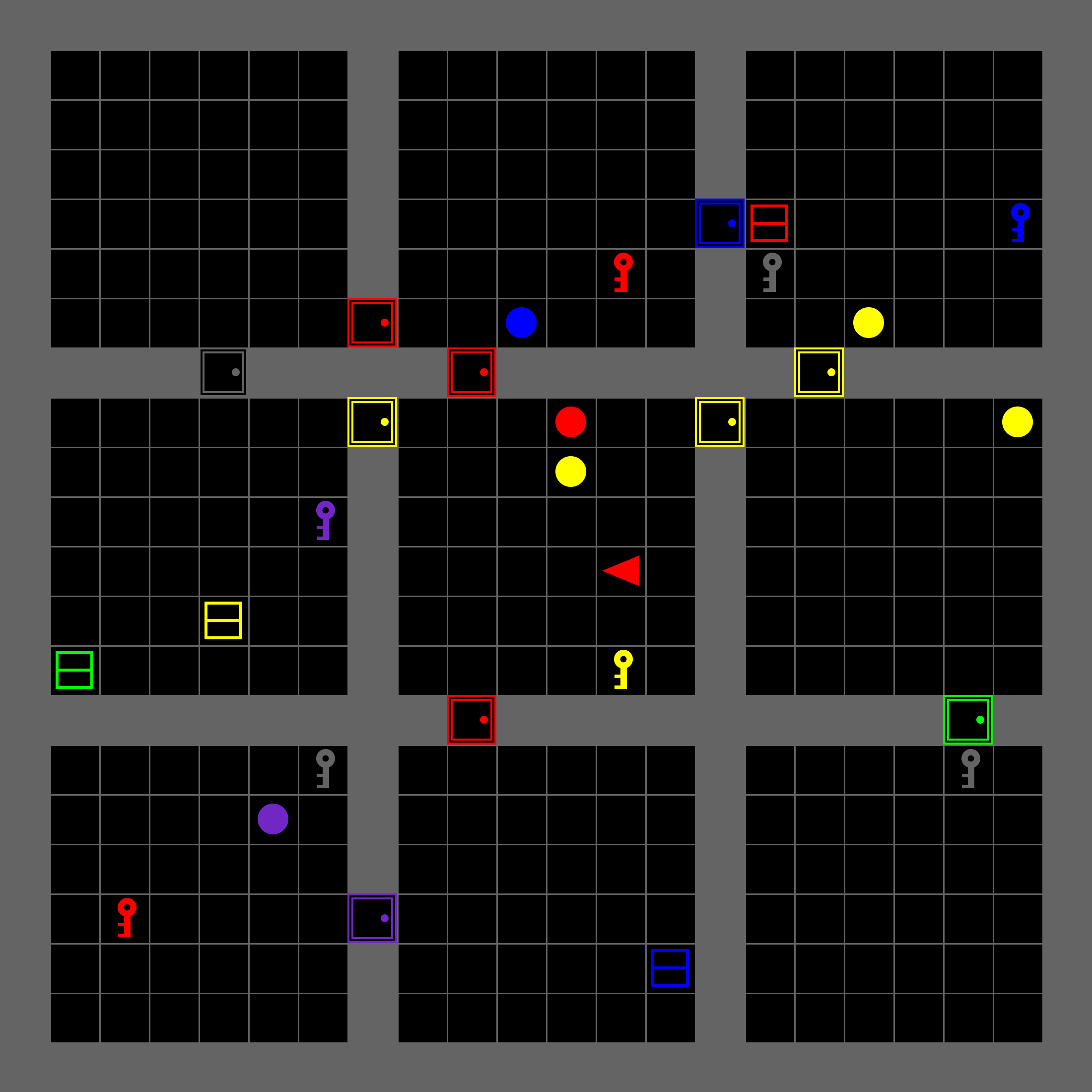}
    \caption{State-dependent action redundancy tasks.
    From left to right are PutNextLocal, Unlock Pickup, Maze 4 Rooms, and Maze 9 Rooms.
    There are common identifiers among these tasks: agent~(triangle), box~(square), key, ball, and door~(on the wall).
    }
    \label{fig: task3 setting}
\end{figure*}


\section{State-dependent Action Clusters}
    We visualize the learned similarity factor.
The results in Figure~\ref{fig: task3 visual} and Figure~\ref{fig: different action clusters} show that the agent has different action clusters in various state positions, which can be successfully captured based on the learned similarity factor.

\begin{figure*}[h]
    \centering
    \subfloat{\includegraphics[scale=0.30]{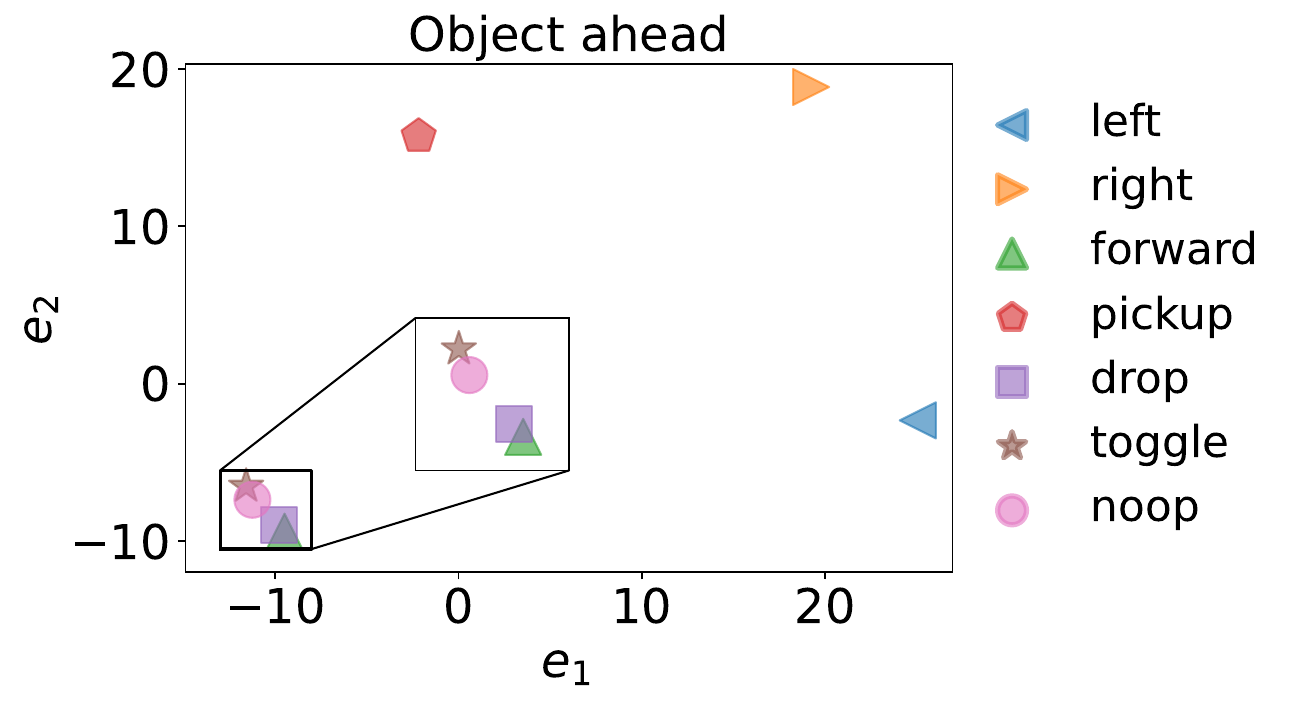}\label{fig: sub_figure1}}
    \subfloat{\includegraphics[width=0.33\hsize, height=0.18\hsize,trim=50 70 50 70,clip]{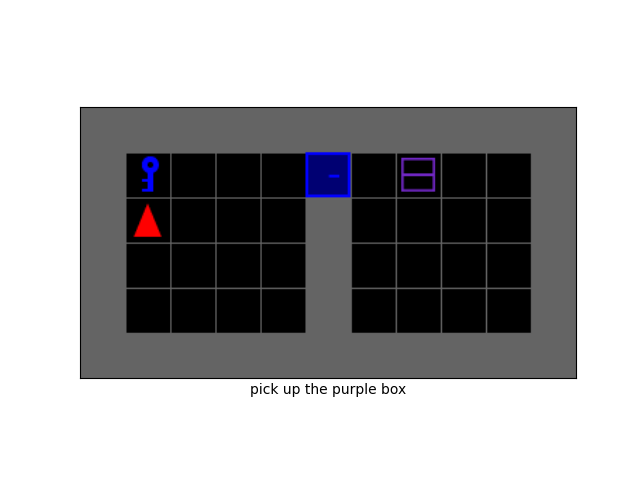}}\label{fig: sub_figure1}
    \vspace{0.5cm}
    
    \subfloat{\includegraphics[scale=0.30]{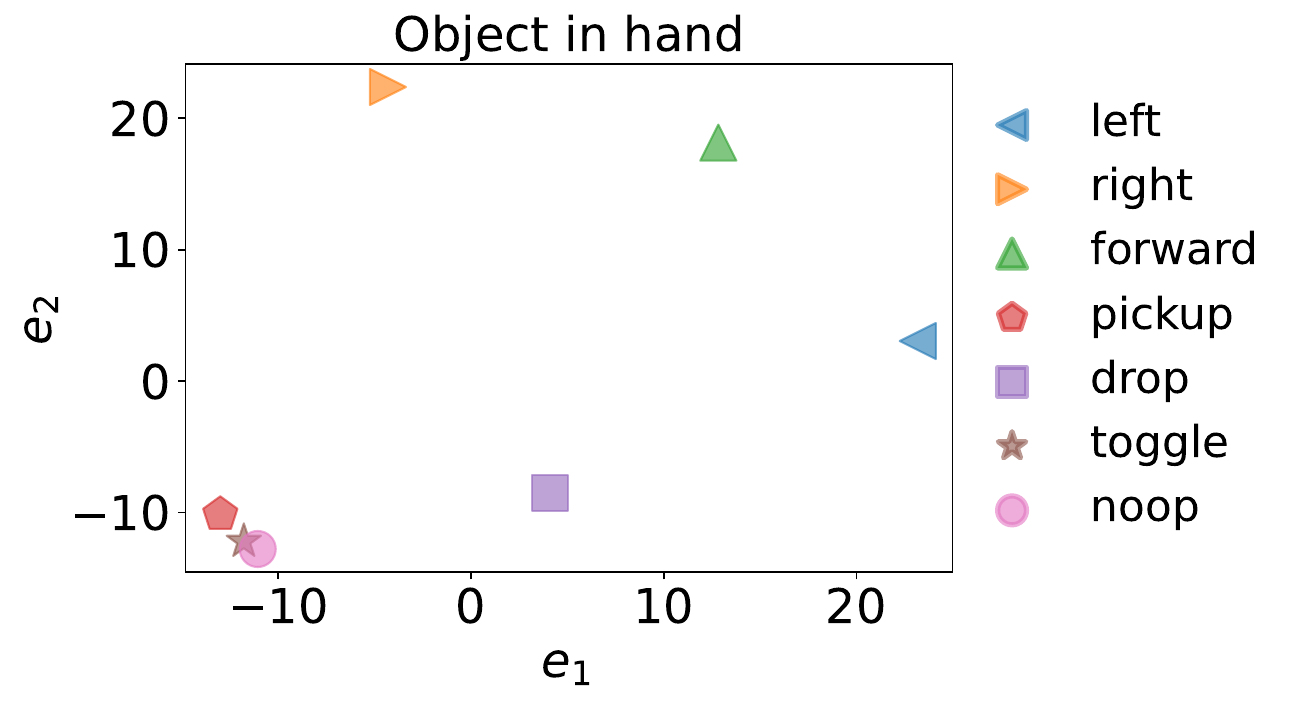}\label{fig: sub_figure1}}
    \subfloat{\includegraphics[width=0.33\hsize, height=0.18\hsize,trim=50 70 50 70,clip]{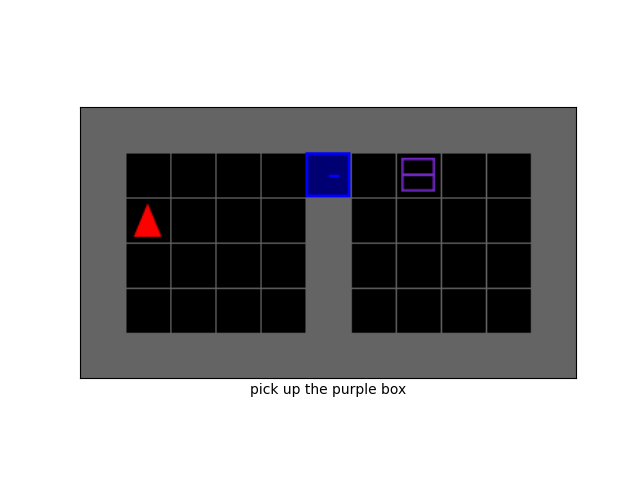}}\label{fig: sub_figure1}
    \vspace{0.5cm}
    
    \caption{Left: Visualization of similarity factor. 
    Right: Two different state positions of the agent in the UnlockPickup task.
    We find that agents in different positions have various action clusters.
    }
    \label{fig: task3 visual}
\end{figure*}

\begin{figure*}[h]
    \centering
    \subfloat{\includegraphics[scale=0.30]{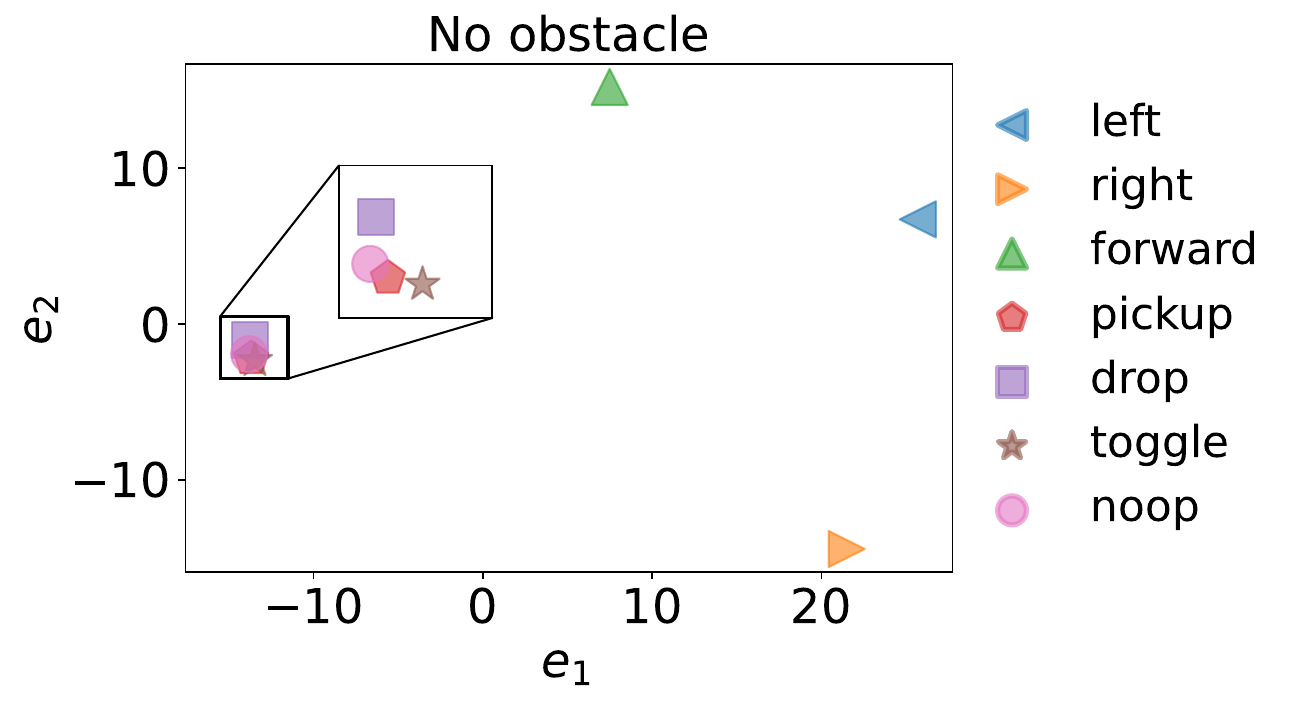}\label{fig: sub_figure1}}
    \subfloat{\includegraphics[width=0.33\hsize, height=0.18\hsize,trim=50 70 50 70,clip]{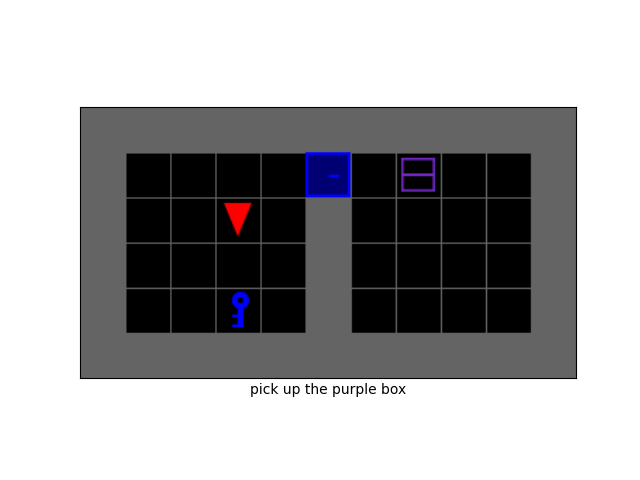}\label{fig: sub_figure1}}
    \vspace{0.5cm}
    
    \subfloat{\includegraphics[scale=0.30]{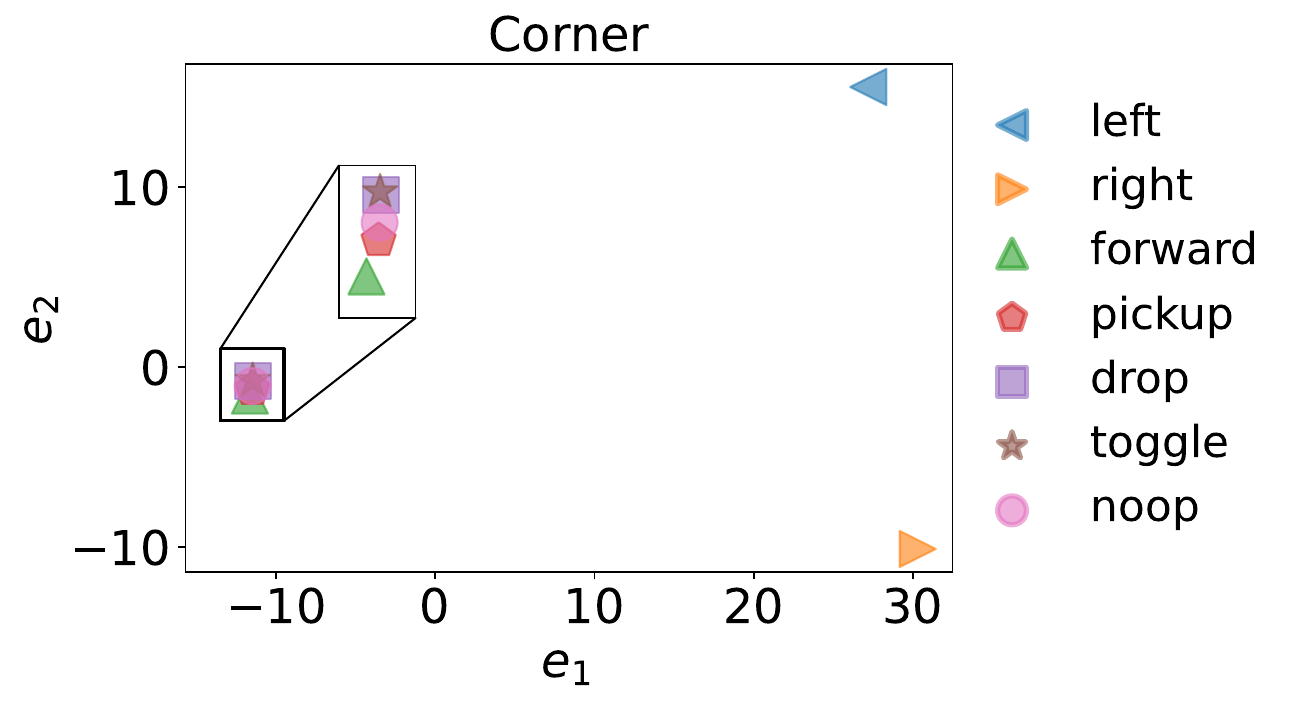}\label{fig: sub_figure1}}
    \subfloat{\includegraphics[width=0.33\hsize, height=0.18\hsize,trim=50 70 50 70,clip]{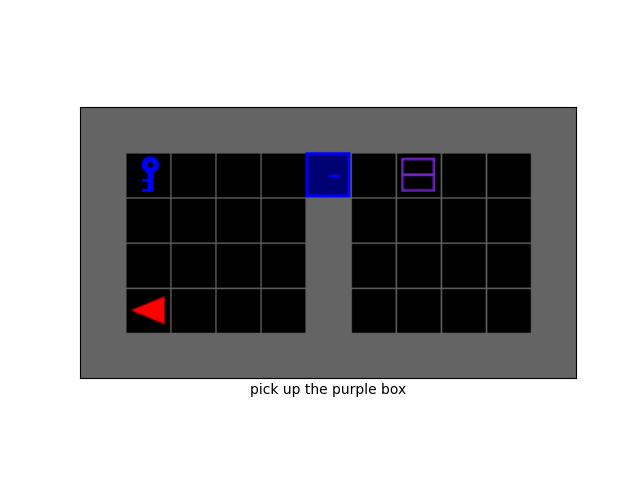}\label{fig: sub_figure1}}
    \vspace{0.5cm}
    
    \subfloat{\includegraphics[scale=0.30]{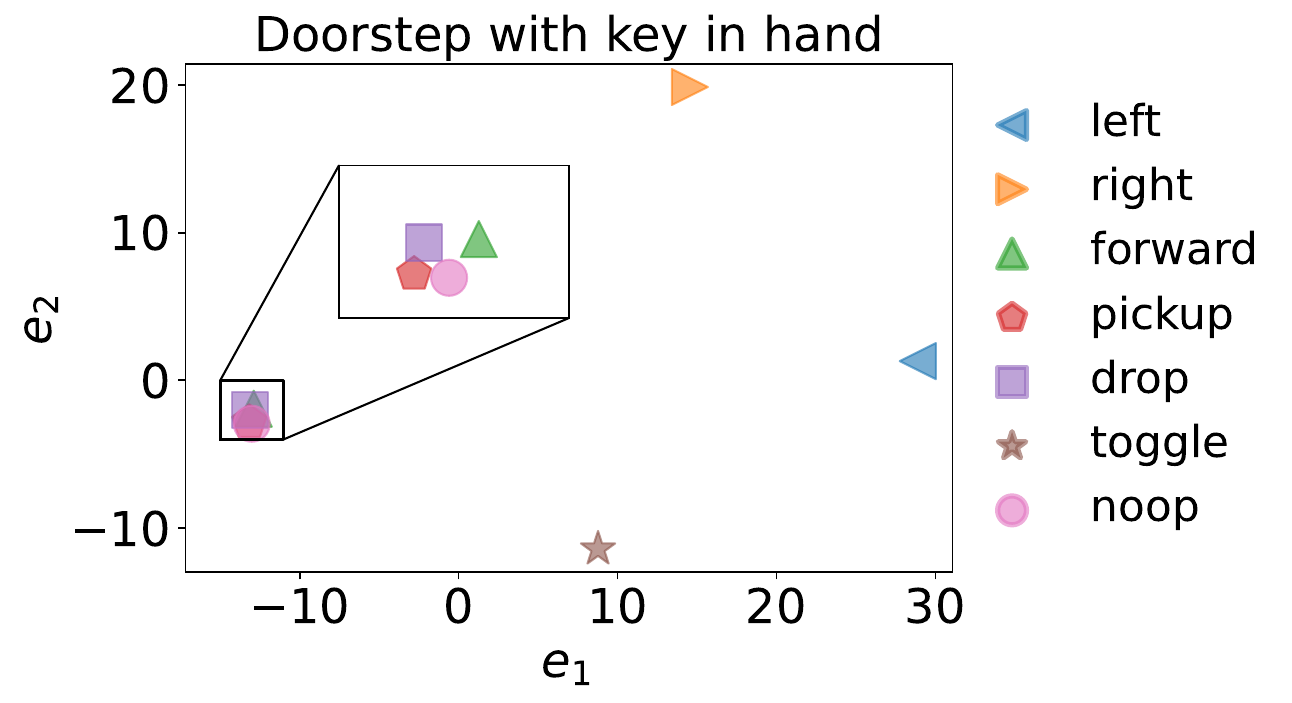}\label{fig: sub_figure1}}
    \subfloat{\includegraphics[width=0.33\hsize, height=0.18\hsize,trim=50 70 50 70,clip]{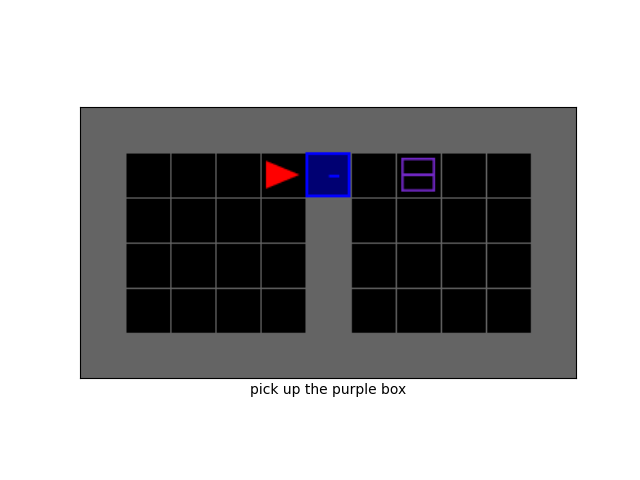}\label{fig: sub_figure1}}
    \vspace{0.5cm}
    
    \subfloat{\includegraphics[scale=0.30]{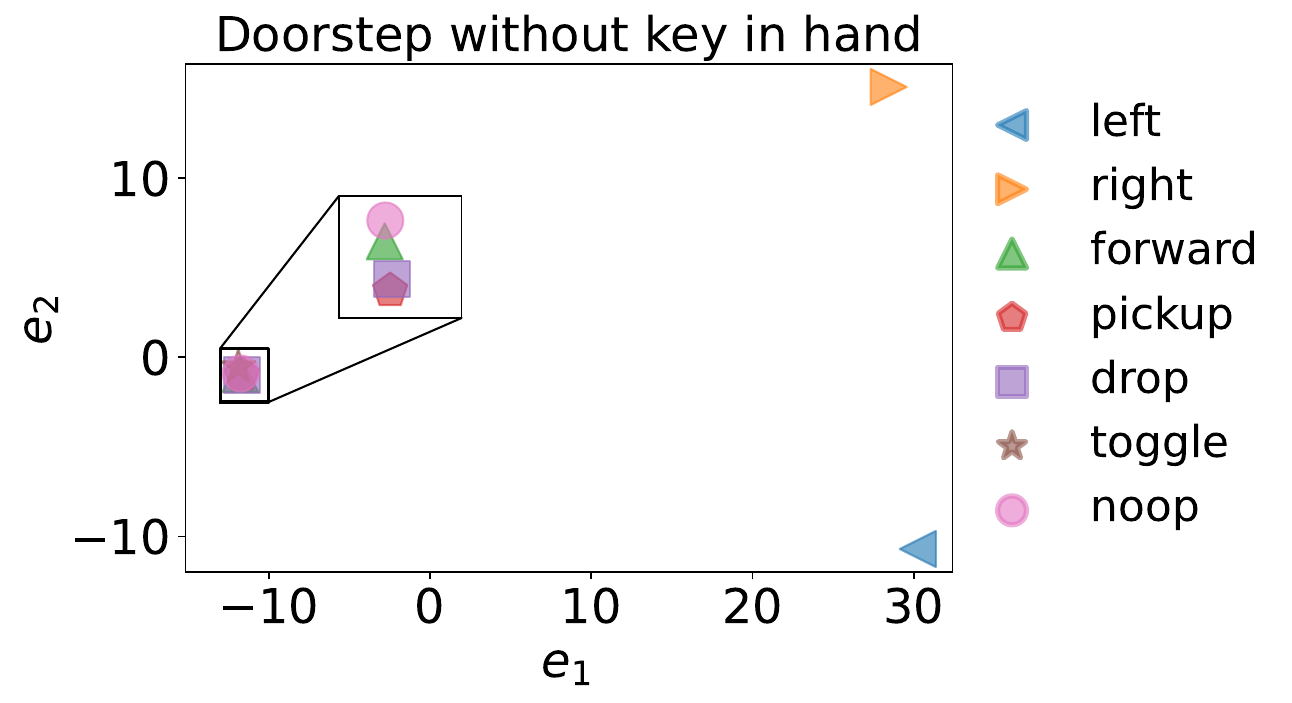}\label{fig: sub_figure1}}
    \subfloat{\includegraphics[width=0.33\hsize, height=0.18\hsize,trim=50 70 50 70,clip]{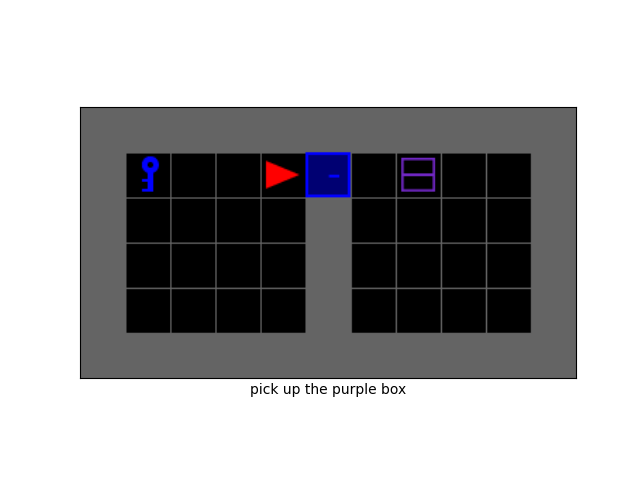}\label{fig: sub_figure1}}
    \vspace{0.5cm}
    
    \caption{Left: Visualization of similarity factor. 
    Right: Four More different state positions of the agent in the UnlockPickup task.
    We find that agents in different positions have various action clusters.
    }
    \label{fig: different action clusters}
\end{figure*}

\clearpage
\section{Atari Benchmark}
\label{appendix: atari}

We conduct experiments on Atari benchmark to provide evidence for NPM’s scalability to high-dimensional pixel-based observations. 
As illustrated in Figure~\ref{fig:DemonAttack cluster} in Atari 2600's Demon Attack, when the agent can only move horizontally and has no bullets, not only does NOOP exhibit redundant actions, but the action set also contains many actions that resemble LEFT or RIGHT actions. 
For instance, there are actions like LEFTFIRE and UPLEFT, which are similar to LEFT over this state.
The results shed light on the clustering of actions into three distinct categories, namely NOOP, LEFT, and RIGHT.

\begin{figure*}[h]
\centering\subfloat{\includegraphics[scale=0.4]{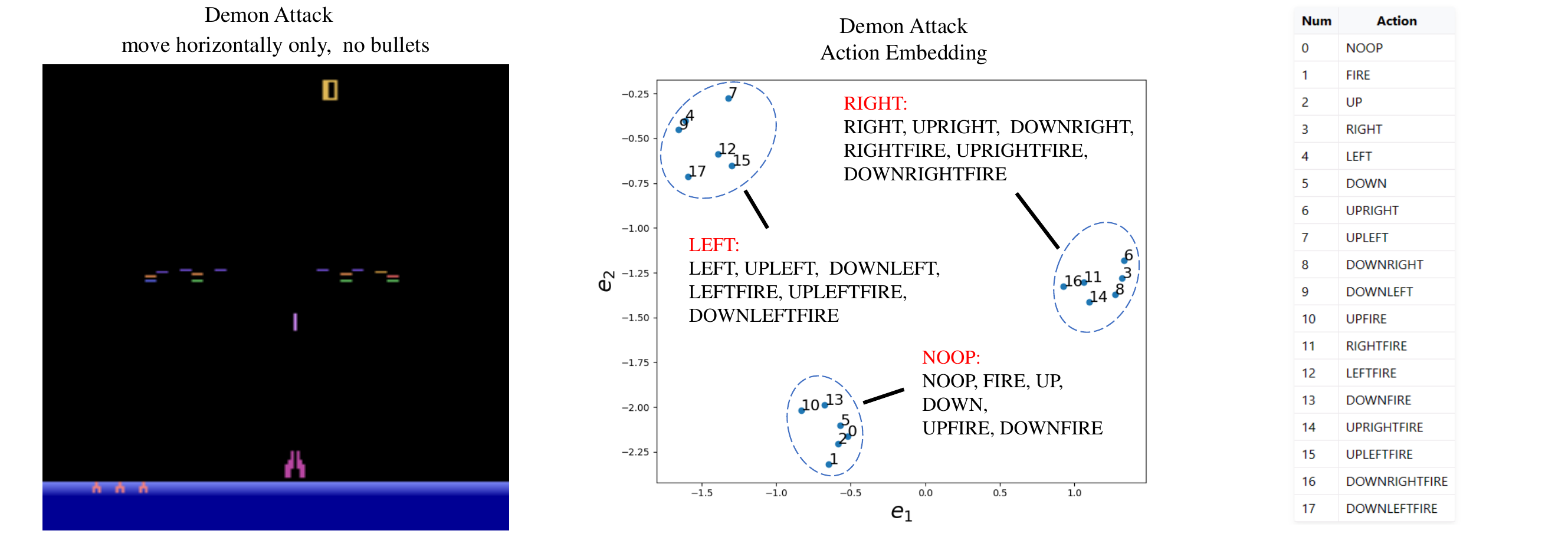}}
\caption{Left: The state where the agent can only move horizontally and has no bullets. Middle: The t-SNE projection based on similarity factor over this state. Right: The original action space of Atari2600 Demon Attack.}
\label{fig:DemonAttack cluster}
\end{figure*}

Benefiting from the reduction in action space brought about by action clustering in Figure~\ref{fig:DemonAttack cluster}, NPM outperforms the baseline method as Figure~\ref{fig: Atari result} shown. We conduct experiments on Atari benchmark with online action aggregation. While two phase training sheds of light on the generality of action masks, incorporating mask training and policy training is also feasible. 
However, it is important to note that during the online use of the mask, the aggregation of actions may not be accurate especially at the beginning. Intuitively, in the modified algorithm, we utilize the mask mechanism to accelerate exploration and training only over states where the agent has been visited a sufficient number of times (e.g. lower prediction errors in random network distillation~\cite{burda2018exploration}).

\begin{figure*}[h]
\centering\subfloat{\includegraphics[scale=0.22]{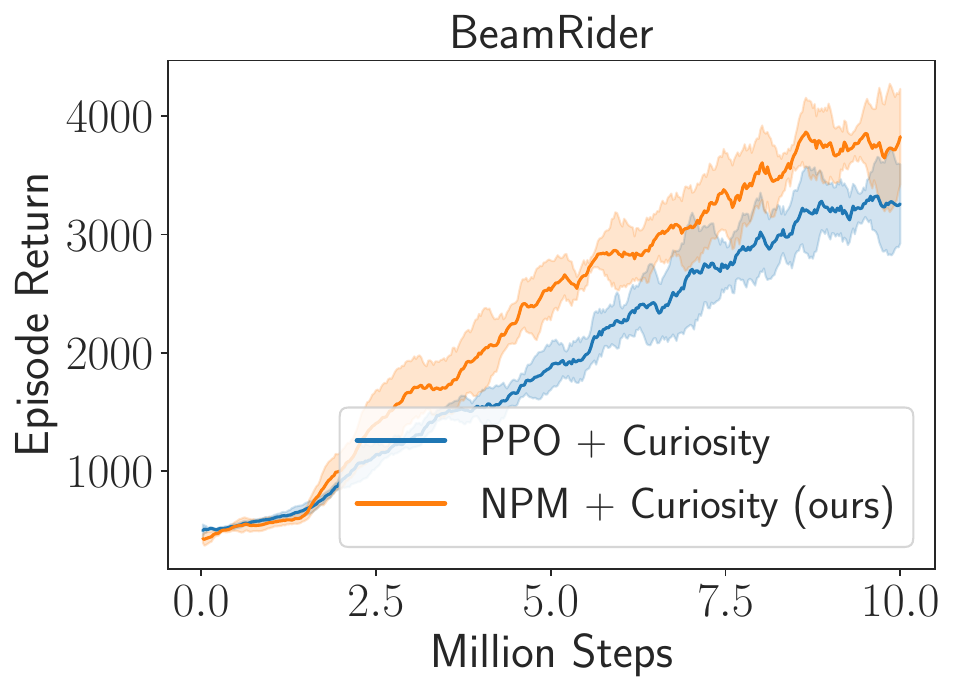}}\subfloat{\includegraphics[scale=0.22]{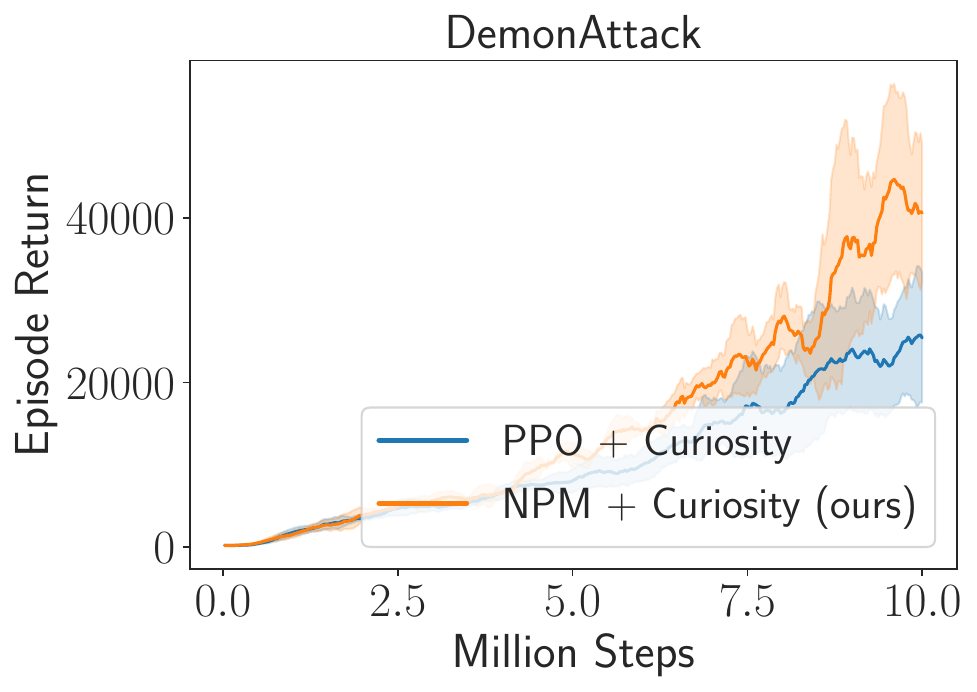}}\subfloat{\includegraphics[scale=0.22]{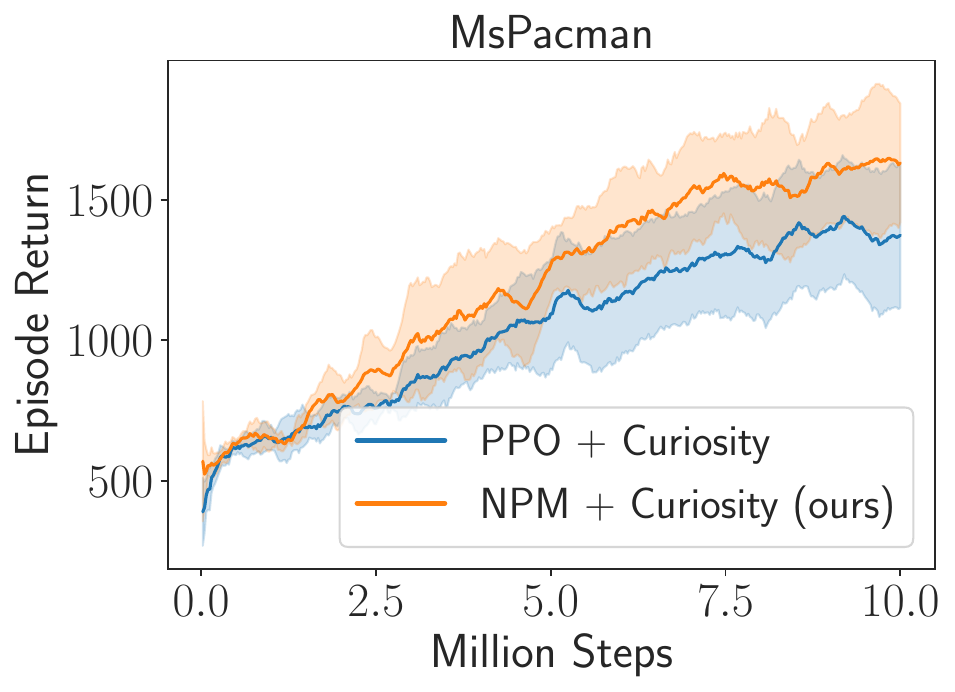}}\subfloat{\includegraphics[scale=0.22]{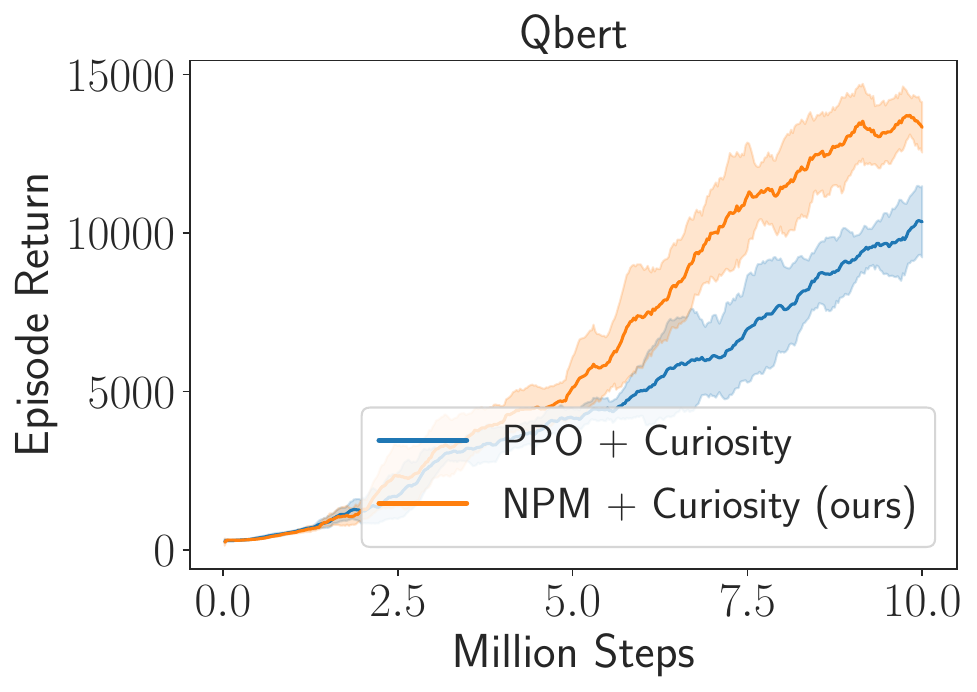}}
\caption{The experimental results on the Atari2600 benchmarks under five random seeds.}
\label{fig: Atari result}
\end{figure*}

\clearpage
\section{Experimental Details} 
\label{appendix: exp detail}
The hyper-parameter of our method is shown in Table~\ref{NPM-PPO hyper-parameters sheet} and Table~\ref{NPM-DQN hyper-parameters sheet}.
Both N-value network, the inverse dynamic model and the policy network adopt the same neural network's hidden layer [64,64].
The similarity factor threshold $\epsilon$ is searched over $\{0.01,0.05, 0.1, 0.2,0.3,0.4,0.5\}$.

\begin{table}[h]
		\caption{NPM-PPO Hyper-parameters sheet}
		\label{NPM-PPO hyper-parameters sheet}
		\centering
		\begin{tabular}{ll}
			\toprule
			Hyper-parameter & Value \\
			\midrule
			Shared & \\
			\hspace{0.3cm} Policy network learning rate & $3 \times 10^{-4}$\\
            \hspace{0.3cm} N-Value network learning rate & $3\times10^{-4}$\\
            \hspace{0.3cm} N-Value network update interval $T$ & $1$\\
			\hspace{0.3cm} Optimizer & Adam \\
			\hspace{0.3cm} Discount factor $\gamma$ & 0.99 \\
        \hspace{0.3cm} Entropy coefficient $\gamma$ & 0.20 \\
			\hspace{0.3cm} Parameters update rate $d$ & 2048 \\
                \hspace{0.3cm} Hidden dimension & [64, 64] \\
			\hspace{0.3cm} Activation function & Tanh \\
			\hspace{0.3cm} Batch size & 64 \\
			\hspace{0.3cm} Replay buffer size & $5\times 10^4$ \\
			\midrule
			Others & \\
			\hspace{0.3cm} $\epsilon$ & $\{0.01,0.05, 0.1, 0.2,0.3,0.4,0.5\}$ \\
			\bottomrule
		\end{tabular}
	\end{table}

\begin{table}[h]
		\caption{NPM-DQN Hyper-parameters sheet}
		\label{NPM-DQN hyper-parameters sheet}
		\centering
		\begin{tabular}{ll}
			\toprule
			Hyper-parameter & Value \\
			\midrule
			Shared & \\
			\hspace{0.3cm} Learning rate & $1 \times 10^{-4}$\\
            \hspace{0.3cm} N-Value network learning rate & $3\times10^{-4}$\\
            \hspace{0.3cm} N-Value network update interval $T$ & $1$\\
			\hspace{0.3cm} Buffer size & $1\times10^{6}$\\
			\hspace{0.3cm} Discount factor $\gamma$ & 0.99 \\
			\hspace{0.3cm} Batch size & 32 \\
                \hspace{0.3cm} Exploration initial value ($\epsilon$-greedy) &1.0 \\
                \hspace{0.3cm} Exploration final value ($\epsilon$-greedy) &0.05 \\
                \hspace{0.3cm} Learning starts & $5\times 10^4$ \\
                \hspace{0.3cm} Target update interval & 200\\
			\midrule
			Others & \\
			\hspace{0.3cm} $\epsilon$ & $\{0.01,0.05, 0.1, 0.2,0.3,0.4,0.5\}$ \\
			\bottomrule
		\end{tabular}
	\end{table}
 
\paragraph{Baseline Hyper-parameter:}
As a thumb rule mentioned in ~\cite{baram2021action}, the action redundancy coefficients $\lambda$ is searched over  $[0.0005,0.5]$ for different environments.

\begin{table}[h]
		\caption{MinRed Hyper-parameters sheet}
		\label{MinRed hyper-parameters sheet}
		\centering
		\begin{tabular}{ll}
			\toprule
			Hyper-parameter & Value \\
			\midrule
			Shared & \\
			\hspace{0.3cm} Policy network learning rate & $3 \times 10^{-4}$\\
            \hspace{0.3cm} N-Value network learning rate & $3\times10^{-4}$\\
            \hspace{0.3cm} N-Value network update interval $T$ & $1$\\
			\hspace{0.3cm} Optimizer & Adam \\
			\hspace{0.3cm} Discount factor $\gamma$ & 0.99 \\
                \hspace{0.3cm} Entropy coefficient $\gamma$ & 0.15 \\
			\hspace{0.3cm} Parameters update rate $d$ & 2048 \\
                \hspace{0.3cm} Hidden dimension & [64, 64] \\
			\hspace{0.3cm} Activation function & Tanh \\
			\hspace{0.3cm} Batch size & 64 \\
			\hspace{0.3cm} Replay buffer size & $5\times 10^4$ \\
			\midrule
			Others & \\
			\hspace{0.3cm} $\lambda$ & $\{0.0005,0.001,0.005,0.01,0.05,0.1,0.5\}$ \\
			\bottomrule
		\end{tabular}
	\end{table}

\paragraph{Pre-training Steps in Phase 1:}
The number of pre-training steps to train the inverse dynamics model in our experiment is shown in the Table~\ref{The number of pre-training steps in phase 1}. The results show that as the action space size grows, so do the training steps. Moreover, if action redundancy is state-dependent and certain states are difficult to reach in the domain (e.g., Minigrid and Atari Benchmark), the training steps should be larger to learn a good model that applies to broader state space.

Furthermore, it should be noted that the phase 1 pre-training of NPM (without extrinsic rewards) is different from policy training in phase 2 (with extrinsic rewards). Compared with the policy training, the phase 1 pre-training of NPM is task-agnostic and reward-free, which sheds light on the generality of action masks. As the experiments on the Maze 4/9 Rooms environment (Figure 9 in the original paper) shown, a single mask trained for one task (e.g., Go To Green Box) can be seamlessly transferred to multiple other tasks (e.g., Go To Red Ball and Go To Grey Key).
\begin{table}[h]
		\caption{The number of pre-training steps in phase 1}
		\label{The number of pre-training steps in phase 1}
		\centering
		\begin{tabular}{llll}
			\toprule
			Environment & Size of Action Space & Property & Training Steps \\
			\midrule
                Four-Room Task & 4, 11, 19, 35 & Synthetic	& 50,000 \\
                Maze Task	& 16, 64& Combined Action	& 50,000 \\
                Maze Task	&256	&Combined Action	&100,000 \\
                Minigrid	& 7	&State-Dependent	&500,000 \\
                Atari Benchmark	&18	&State-Dependent	&1,000,000 \\
			\bottomrule
		\end{tabular}
	\end{table}
\end{document}